%% file: neurips_2026.tex
\documentclass{article}

\PassOptionsToPackage{numbers,compress}{natbib}

\usepackage[preprint]{neurips_2026}

\usepackage[utf8]{inputenc} 
\usepackage[T1]{fontenc}    
\usepackage{hyperref}       
\usepackage{url}            
\usepackage{booktabs}       
\usepackage{amsfonts}       
\usepackage{nicefrac}       
\usepackage{microtype}      
\usepackage[dvipsnames]{xcolor}
\usepackage{graphicx}
\usepackage{comment}
\usepackage{multirow}
\usepackage{makecell}
\usepackage{subfigure}

\input{defs-kjun-v6-overleaf.tex}

\usepackage{quoting}
\quotingsetup{vskip=0pt}

\usepackage[toc,page,header]{appendix}
\usepackage{minitoc}
\usepackage{silence}
\usepackage{algorithm}
\usepackage{algorithmic}

\newtcolorbox{highlight}[1][]{
    colback=gray!10,
    colframe=gray!30,
    boxrule=0.5pt,
    arc=2pt,
    leftrule=3pt,
    rightrule=3pt,
    toprule=1pt,
    bottomrule=1pt,
    #1
}

\usepackage[capitalize,noabbrev]{cleveref}

\usepackage[textsize=tiny]{todonotes}

\title{Beyond RLHF: A Unified Theoretical Framework of Alignment}

\author{%
  Jihun Yun$^1$\thanks{Equal contribution}
  \qquad 
  Juno Kim$^2$\footnotemark[1]\, \thanks{Work done during an internship at KRAFTON.}
  \qquad  
  Jongho Park$^{1,2}$
  \qquad
  Junhyuck Kim$^1$\\ 
  \bfseries
  Jongha Jon Ryu$^3$
  \qquad 
  Jaewoong Cho$^1$
  \qquad 
  Kwang-Sung Jun$^4$\\
  $^1$KRAFTON\qquad
  $^2$UC Berkeley\qquad
  $^3$MIT\qquad
  $^4$POSTECH\\
  \texttt{jihuny@krafton.com\qquad
  junokim@berkeley.edu\qquad
  kwangsungjun@postech.ac.kr
  }
}

\begin{document}

\maketitle

\begin{abstract}
    Alignment via reinforcement learning from human feedback (RLHF) has become the dominant paradigm for controlling the quality of outputs from large language models (LLMs).
    However, existing theories do not provide strong justification for the RLHF objective itself and do not allow comparisons of the guarantees between various methods because different methods are often analyzed under different frameworks.
    Toward a unified framework for alignment, we ask under what assumptions can we derive existing or new training objectives and obtain theoretical guarantees.
    To this end, we reframe alignment as \emph{distribution learning} from pairwise preferences, which makes a probabilistic assumption describing how preferences reveal information about the target LM.
    This leads us to propose three principled alignment objectives: preference maximum likelihood estimation, preference distillation, and reverse KL minimization.
    We prove that they all enjoy strong non-asymptotic $O(1/n)$ convergence to the target LM, naturally avoiding degeneracy.
    In particular, reverse KL highly resembles the RLHF objective, providing strong justification for RLHF.
    Furthermore, our theory explains, for the first time, the empirical finding that on-policy objectives (e.g., RLHF) typically outperform likelihood-style objectives (e.g., DPO). Finally, empirical results indicate that the proposed objectives are competitive with strong baselines across several tasks and models.
    \looseness=-1
\end{abstract}

\def\piref{{\pi_0}}

\section{Introduction}
\label{sec:intro}

Alignment refers to the task of improving the quality of responses (e.g., helpfulness and harmlessness) generated from large language models (LLMs) via human preferences~\citep{bai22training,ouyang22training} and has become the de facto final step in LLM training.
The first method introduced for alignment is Reinforcement Learning from Human Feedback (RLHF)~\citep{christiano17deep,stiennon20learning}, which trains a reward model $R$ from pairwise preferences and then optimizes a policy $\pi$ (i.e., language model) that maximizes the reward via reinforcement learning (RL):
\begin{align}\label{eq:rlhf}
   \max_{\pi}~~ \EE_{x\sim \cD, a\sim\pi(x)}\big[\hR(x,a)\big] - \beta \KL(\pi\Vert\pi_0)
\end{align}
where $\hR$ is a reward function learned from preference data, $\cD$ is the prompt distribution, $\pi(x)$ is the policy $\pi$'s response distribution to prompt $x$, $\KL(p \Vert q) = \EE_{x\sim\cD}\left[\KL(p(x)\Vert q(x))\right]$ is the Kullback-Leibler divergence (between policies), $\piref$ is a supervised fine-tuned reference LLM, and $\beta>0$ is the regularization strength.

The RLHF objective is central to various practical algorithms and has fundamentally shaped how researchers think about alignment.
For example, DPO (Direct Policy Optimization) can be viewed as reformulating RLHF under the strong `all-policy' assumption so the objective consists of simple likelihood ratio terms rather than relying on on-policy responses~\citep{rafailov23direct}.
$\Psi$PO extends RLHF by generalizing $R(x,a)$ to a $\Psi$-transformation of the preference probability~\citep{azar24general}.
Even theoretical analyses of alignment algorithms often treat the RLHF objective or its variants as the ultimate learning-theoretic goal and prove convergence guarantees such as regret bounds \citep{zhan2024provable,xiong24iterative,zhang2024self,huang25correcting,xie25exploratory}.
Certainly, the RLHF objective has proven useful, and one may argue that it is a sensible objective.

\begin{table*}[t!]
    \centering
    \caption{Summary of our proposed methods and theoretical guarantees. In each section, we draw parallels to existing approaches such as DPO and REBEL~\cite{gao24rebel}.}\label{tab:tasks}
    \resizebox{\linewidth}{!}{%
    \begin{tabular}{@{}l|ccccc@{}}
    \toprule
        Distribution Learning & Related to & \makecell{Reward \\Model}  & \makecell{Requires \\ RL Training} & Objective & \makecell{Forward KL\\Guarantee } \\ \midrule \midrule
        \textbf{Preference MLE} (Sec.~\ref{sec:mle}) & DPO & Not Used & No & Eq.~\ref{eq:mle-nokl} & $O(1/n)$ (Thm.~\ref{thm:mle}) \\
        \textbf{Preference distillation} (Sec.~\ref{sec:pref}) & \makecell{REBEL} & Required & No & Eq.~\ref{eq:distill} & $O(1/n)$ (Thm.~\ref{thm:pref_distill}) \\
        \textbf{Reverse KL} (Sec.~\ref{sec:revkl}) & RLHF & Required & Yes & Eq.~\ref{eq:rkl} & $O(1/n)$ (Thm.~\ref{thm:reverse_kl}) \\
        \bottomrule
    \end{tabular}%
    }
    \vskip -10pt 
\end{table*}

However, the RLHF objective has no known justification from a learning-theoretic viewpoint.
First, existing theoretical guarantees typically show convergence to the solution of the population RLHF objective with the true reward (see the works cited above).
This is a tautology and can hardly be interpreted as a justification.
Second, while maximizing rewards \textit{informally} sounds reasonable, it is not clear what it \textit{formally} means because the problem definition of alignment does not involve rewards at all! Third, the \emph{a priori} dependence on reward manifests in the following issue: \eqref{eq:rlhf} is a standard machine learning objective of the form `loss + regularizer' where the loss (on-policy negative reward) learns from data and the regularizer leverages an available starting model.
With sufficient data, the regularization strength $\beta$ should reduce, and in the asymptotic regime, should become zero to ensure sufficient learning.
In this sense, asymptotically, the RLHF objective faces a dilemma: it must either tend to zero regularization so that the solution to \eqref{eq:rlhf} becomes deterministic (undesirable as a language model), or use nonzero regularization and hinder the learning process.
This suggests that, in principle, RLHF requires some correction.

Then, what kind of justification should we require? One strong form of justification is provable guarantees such as statistical convergence to the target language model (LM).
However, existing analyses are typically only applicable to their method of choice. For example, theoretical frameworks for DPO \citep{agarwal25design,kveton25active} and those for RLHF \citep{zhan2024provable,xiong24iterative,zhang2024self,huang25correcting,xie25exploratory} are incompatible (i.e., have different theoretical targets), preventing us from directly comparing their guarantees.

The limitations discussed above call for a unified learning-theoretic framework for alignment. As such, we ask:
\begin{highlight}
    \begin{quoting}[leftmargin=0.5em, rightmargin=0.5em]
    \noindent \textit{Under what assumptions (e.g., how the target LM is related to preference data) can we theoretically justify existing training objectives or even develop new ones?}
    \end{quoting}
\end{highlight}

In this paper, we take the popular learning-theoretic treatment where we first posit the existence of the target model we aim to learn and then make probabilistic assumptions about the data-generating process as a function of the target model, followed by developing objectives based on statistical principles.
This treatment is widely used in classification (e.g., logistic regression), topic modeling (e.g., latent Dirichlet allocation), and generative models (e.g., variational autoencoder).
This offers two key benefits: (i) explicit modeling assumptions which often help understand model behavior, and (ii) rigorous learning-theoretic guarantees such as consistency -- convergence to the target model as the sample size grows.
Perhaps surprisingly, such a \emph{fully probabilistic} framework for alignment remains largely unexplored.
See Appendix \ref{app:related} for related work.
\looseness=-1

\paragraph{Our contributions.}
We move beyond blindly taking the RLHF objective as the ultimate goal. We propose a novel unified theoretical framework for alignment that can be seen as a fully probabilistic approach to distribution learning from pairwise preference data, without any reliance on an \emph{a priori} notion of reward maximization. Specifically, we posit that there exists a target (oracle) LM $\pi^*$ and explicitly model how information about $\pi^*$ is revealed through preference feedback.
Intuitively, $\pi^*$ must \textit{assign a higher probability to the preferred response}, which we encode as the assumption
\begin{align}\label{eq:preference-model-intro}
  \PP(a \succ b \mid x) = \fr{\pi^*(a\mid x)^\gam}{\pi^*(a \mid x)^\gam + \pi^*(b\mid x)^\gam}
\end{align}
for some $\gam>0$ where $a \succ b$ means the response $a$ is preferred over $b$.
This is an instance of a Bradley-Terry (BT) model~\citep{bradley52rank} with preference score being the tilted response probability of $\pi^*$.
The main difference from the BT model in RLHF is that our assumption says that the preference model directly depends on the target LM rather than some reward function.
Note that our assumption says that the preference model is \emph{explicitly} a language model.
This is in stark contrast to DPO which starts from the RLHF formulation and leverages the (unrealistic) all-policy assumption to realize that there is a \emph{secret} relationship between the reward model (or preference model) and the LM~\citep{rafailov23direct}.

Our simple assumption leads to various training objectives that are relatable to existing works, and solutions of them provably converge to $\pi^*$ in terms of the KL divergence.
Specifically, we propose the following three algorithms (summarized in Table~\ref{tab:tasks}):
\begin{itemize}[leftmargin=4mm,topsep=0pt]
    \parskip=.3em
    \itemsep=0em
    \item \textbf{PMLE} (Preference Maximum Likelihood Estimate; Section~\ref{sec:mle}): This objective maximizes the likelihood of the preference model \eqref{eq:preference-model-intro}, subject to reverse KL regularization w.r.t. a reference policy $\piref$. Similarly to DPO, it is relatively straightforward to optimize.
    \item \textbf{Preference distillation} (Section~\ref{sec:pref}):
    By directly estimating the expected preference from a learned reward model, the MLE can be rewritten as distilling the preference distribution into a language model. Unlike existing reward distillation~\citep{fisch25robust,gao24rebel}, this formulation is explicitly derived from the Bradley-Terry model~\eqref{eq:preference-model-intro}.
    \item \textbf{Reverse KL} (RKL; Section~\ref{sec:revkl}):
    Since our goal is distribution learning, it is natural to optimize the reverse KL divergence $\EE_x[\KL(\hpi(x) \Vert \pi^*(x)]$. Although $\pi^*$ is unknown, its unnormalized form can be estimated from \eqref{eq:preference-model-intro} with a shallow network.
    Plugging in this estimate and adding a KL regularizer lead to a variant of the RLHF objective with an additional entropy term, which can be seen as a correction to RLHF. Figure~\ref{fig:rlhf_comparison} shows that under our preference model, RLHF with tuned $\beta$ for each sample size $n$ indeed suffers from inconsistency (i.e., does not converge to the true $\pi^*$) where as RKL is consistent.
    \item \textbf{Theoretical guarantees}: For all three algorithms, we prove upper bounds on the forward KL error of the form: $\EE_x[\KL(\pi^*(x)\Vert \hpi(x))] \le O(1/n)$ where $n$ is the size of the preference dataset.
\end{itemize}

\begin{wrapfigure}{r}{0.4\linewidth}
    \centering
    \vspace{-1.0em}
    \includegraphics[width=\linewidth]{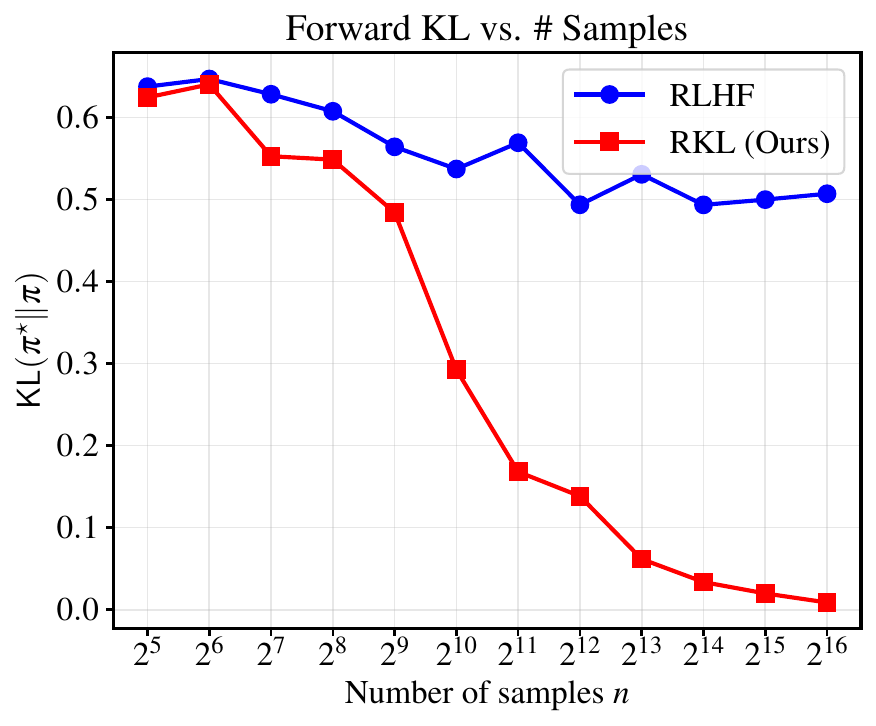}
    \vspace{-2.0em}
    \caption{RLHF vs. RKL.}
    \label{fig:rlhf_comparison}
    \vskip -10pt
\end{wrapfigure}
Our guarantees are non-asymptotic and first-of-its-kind for learning a distribution from pairwise feedback, to the best of our knowledge~\citep[cf.][]{dumoulin23density}.
Furthermore, our theoretical framework thus offers a novel statistical grounding for RLHF up to a minor correction, without resorting to the tautological justification (i.e., convergence to the population version of the training objective).
Interestingly, our guarantee suggests that RKL (similar to RLHF) has a superior guarantee compared to PMLE (similar to DPO), which can be seen as, for the first time in the literature, a theoretical confirmation of the common empirical finding that RLHF outperforms DPO.
Again, such a confirmation could not have been made with existing theories because they provide ad hoc guarantees for each method.

We complement our theory with experiments showing that our methods generally outperform baseline win-rates in TL;DR summarization and generate more preferred responses in general chat scenario; see Section~\ref{sec:exp_main} for details.

\section{Preliminaries}
\label{sec:prelim}

\paragraph{Alignment as distribution learning.} Let $\cX$ and $\cA$ be the space of prompts and responses, respectively, and let $\cD\in\Delta(\cX)$ be a fixed distribution over prompts. We define a language model (LM) as a function or policy $\pi:\cX\to\Delta(\cA)$ determining a collection of conditional (i.e., contextual) distributions $\pi(\cd\mid x)$, which we also denote more simply as $\pi(x)$.\footnote{
  This definition can also cover unconditional distributions by introducing a member in $\cX$ as a null prompt.
}
We view alignment as learning these distributions from pairwise preference feedback, drawn from a model explicitly depending on $\pi^*$, the ideal (target) LM we wish to learn. Hence given a class $\Pi$ of LMs, our ultimate goal is to find $\hpi \in \Pi$ that is as close as possible to $\pi^*$ w.r.t. a suitable measure of distance between distributions.
\looseness=-1

\paragraph{Our preference model.} Let $\mu$ be the LM used for generating responses to be preference-labeled; this could be a reference LLM or simply an existing dataset.
We are given a preference dataset $D_n = \{(x,a^+,a^-)\}$ of $n$ independent samples where $x \sim \cD$ is a prompt and $a^+/a^-$ are preferred/dispreferred responses. We assume that, given $x$, the pair $(a^+,a^-)$ is sampled by drawing responses $a,b\sim \mu(x)$ independently and then sampling a preference from $\PP_*(a\succ b\mid x) := \PP_{\pi^*}(a\succ b\mid x)$, where
\begin{align}\label{eq:preference-model}
\PP_\pi(a \succ b \mid x) := \fr{\pi(a\mid x)^\gam}{\pi(a \mid x)^\gam + \pi(b\mid x)^\gam},
\end{align}
followed by setting $(a^+,a^-)=(a,b)$ if $a$ is preferred over $b$ and $(a^+,a^-)=(b,a)$ otherwise.
The value of $\gamma$ determines the extent to which differences in the response probabilities under policy $\pi$ are accentuated or attenuated. In practice, $\gam$ is a hyperparameter typically set as $0 < \gam < 1$.

\paragraph{Is our preference model too strong?}
One may wonder if \eqref{eq:preference-model-intro} is too restrictive; we claim that it is not. Our assumption (i) takes the BT model $\PP(a\succ b | x) = p^*(a|x)/(p^*(a|x) + p^*(b|x))$
for some underlying preference $p^*$ and then (ii) connects $p^*$ with the target LM $\pi^*$.
For (i), while the BT model has drawbacks such as not allowing cyclic preferences, it is conventional in many prior works
~\cite{christiano17deep,stiennon20learning}.
Also, our distribution learning perspective can be extended to incorporate general preference models.
More importantly, (ii) is the crucial link allowing us to derive guarantees for learning $\pi^*$. Previous theoretical studies provide convergence guarantees for RLHF or other reward-maximizing objectives \citep{zhan2024provable,xiong24iterative,zhang2024self,huang25correcting,xie25exploratory} without justifying why a certain reward construction must be optimized in the first place. In contrast, the main point of our work is to identify under which assumptions such objectives can be justified. Indeed, we will see that objectives derived from our framework often resemble baselines (DPO, RLHF, REBEL), so that we are essentially making the hidden assumptions of existing methods explicit!
\looseness=-1

\paragraph{Theoretical setup.} We call $R_\pi(x,a) := \gam\ln \pi(a\mid x)$ the \emph{reward} induced by $\pi\in\Pi$.\footnote{There is no actual reward in the PMLE scheme; we just call this reward for convenience.} The centered reward is defined as $\bar R_\pi(x,a) := R_\pi(x,a) - \EE_{a\sim \mu(x)}[R_\pi(x,a) \mid x]$. As with $\PP_*$, we write $R_*:=R_{\pi^*}$ and $\bar{R}_* := \bar{R}_{\pi^*}$. Finally, we denote $\Delta \bar{R}_{\pi} := \bar R_{\pi}- \bar R_{*}$.
Our main assumptions, standard in the alignment literature \citep{zhan2024provable,xie2025exploratory,zhang2025distilldatarewardssmaller,agarwal25design,huang25correcting}, are as follows:
\begin{assumption}[Realizability]\label{assp:real}
$\pi^*\in\Pi$ for a finite policy class $\Pi$.
\end{assumption}
\begin{assumption}[Boundedness]\label{assp:bdd}
There exists $R>0$ such that $|\bar{R}_\pi(x,a)| \le \gam R$ for all $\pi\in\Pi$.
\end{assumption}
Since the responses $(a^+,a^-)$ are sampled from $\mu(x)$ rather than $\pi^*(x)$, the alignment problem is an instance of offline learning where there is a distribution shift between the observed data versus the target distribution that we aim to have guarantees on.
It is thus necessary to introduce a coverage assumption between $\mu$ and the policy class $\Pi$, which is well-studied in the offline RL literature~\cite{agarwal19reinforcement}.
In particular, we use the following generalized coverage coefficient~\citep{xie21bellman,agarwal25design}.

\begin{definition}[Generalized coverage coefficient]\label{def:cov}
For a policy class $\Pi'$, we denote by $C_{\Pi'}>0$ the smallest constant satisfying for every $\pi\in\Pi'$,
\begin{align*}
   \EE_{x\sim \cD, a\sim \pi^*(x)}\big[\Delta \bar R_\pi(x,a)^2\big] \le C_{\Pi'} \EE_{x\sim \cD, a\sim \mu(x)}\big[\Delta \bar R_\pi(x,a)^2\big].
  \end{align*}
\end{definition}
This improves upon the all-policy $\ell_\infty$-concentrability condition $\sup_{\pi\in\Pi}\max_{x,a}\tfrac{\pi(a\mid x)}{\mu(a\mid x)} \le C'$~\cite{munos03error} as the former can be bounded even if the latter is infinite, depending on $\cD$ and the reward class $\Delta \bar{R}$.%
\footnote{
    While being beyond our scope, leveraging pessimism~\cite{gabbianelli24importance,huang25correcting,zhan22offline} may further improve the coverage coefficient to a single concentrability coefficient that relies on $\pi^*$ rather than the policy class $\Pi$.
}

\section{Preference Maximum Likelihood Estimation Approach}
\label{sec:mle}

We begin by introducing a maximum likelihood-based objective that can be directly derived from treating alignment as distribution learning from pairwise feedback.
Given a
preference dataset
$D_n = \{(x, a^+, a^-)\}$ as described in Section~\ref{sec:prelim},
we wish to estimate $\pi^*$ by finding a policy $\hpi$ that maximizes the likelihood of observed pairwise preferences under the BT preference assumption \eqref{eq:preference-model}. Concretely, the negative log-likelihood for each pair $(x, a^+, a^-)$ under a candidate policy $\pi$ is:
\begin{align*}
    -\ln \PP_\pi(a^+ \succ a^- \mid x) = -\ln \sigma \left(\gamma \ln \frac{\pi(a^+ \mid x)}{\pi(a^- \mid x)}\right),
\end{align*}
where $\sigma(z) = 1/(1+\exp(-z))$ is the logistic sigmoid. Summing over all preference pairs yields
\begin{align}\label{eq:mle-nokl}
    \!\!\! \cL_{\pmle}(\pi) = \frac{1}{n} \sum_{(x,a^+,a^-)\in D_n} \!\!\!\!\!\!-\ln\sigma \left(\gamma \ln \frac{\pi(a^+ \mid x)}{\pi(a^- \mid x)}\right).
\end{align}
By minimizing $\mathcal{L}_{\pmle}$, we encourage $\pi$ to place higher probability on response $a^+$ relative to $a^-$. Note that in practice, we rarely learn a policy $\pi$ from scratch; instead, we typically optimize a
fine-tuned model, referred to as the \emph{reference policy} $\piref$. Thus, it is natural to introduce a KL penalty that keeps $\pi$ close to $\piref$ for alignment: $\beta\cd \KL(\pi(x) \Vert \piref(x))$. Putting everything together, our \textbf{PMLE} (preference maximum likelihood estimation) objective for distribution learning is
\begin{align}\label{eq:mle-objective}
    \cL_{\pmle,\beta}(\pi) := \cL_{\pmle}(\pi) + \beta \KL(\pi(x) \Vert \piref(x)).
\end{align}
\paragraph{Remark.} Recall that DPO \citep{rafailov23direct} minimizes the objective
\begin{align}\label{eq:dpo}
\sum_{D_n} -\ln\sigma\left(\gam\ln\frac{\pi(a^+ \mid x)}{\pi(a^- \mid x)}- \gam\ln\frac{\piref(a^+ \mid x)}{\piref(a^- \mid x)}\right).
\end{align}
Compared to \eqref{eq:mle-objective}, the DPO objective does not have an explicit regularizer, which could lead to undesirable behaviors if the policy class $\Pi$ is sufficiently expressive.
Specifically, \citet{fisch25robust} prove that DPO may converge to a degenerate distribution. Also, \citet{song24theimportance} show that DPO relies on a strong coverage assumption: if $\piref$ does not fully cover the relevant distribution, DPO can produce out-of-distribution responses, making its reward estimates inaccurate. Unlike RLHF, which has a KL term
to stay within the support of $\piref$, DPO can assign non-zero probability to responses that $\piref$ would never select, undermining performance guarantees.
In contrast, our PMLE objective \eqref{eq:mle-objective} incorporates an explicit KL term that effectively circumvents the aforementioned pitfalls.

\paragraph{Convergence guarantee.}

Under the assumptions in Section \ref{sec:prelim}, we show the bound on the forward KL. Throughout the paper, constants relying only on $R$ are hidden. Throughout, All proofs are deferred to Appendix \ref{sec:appendix_proofs}.
\begin{theorem}\label{thm:mle}
  The PMLE estimate $\hpi = \argmin_{\pi\in\Pi} \cL_{\pmle}(\pi)$ satisfies with probability at least $1-\delta$,
\begin{align}\label{eq:pmle_guarantee}
\EE_{x\sim\cD}\left[\KL(\pi^*(x)\Vert \hpi(x))\right] \lesssim \frac{C_{\Pi}}{\gam^2} \cd \frac{\ln(|\Pi|/\delta)}{n}.
\end{align}
\end{theorem}
The proof, provided in Appendix \ref{app:proof_mle}, is inspired in part by~\citet[proof of Theorem 3.6]{agarwal25design}, but we leverage Schulman's trick~\cite{schulman20kl} followed by a quadratic approximation to obtain a $1/n$ rate rather than $1/\sqrt{n}$ that would be obtained when directly following their proof.
Also note that the left-hand side of \eqref{eq:pmle_guarantee} is equivalent to the KL divergence between the induced \emph{joint} distributions on $\cX\times\cA$: $\KL(\cD(x)\pi^*(a\mid x)\Vert\cD(x)\hpi(a\mid x))$.

We assume $\beta=0$ here and for all analysis in the sequel for simplicity and to demonstrate that the objective derived from purely considering preference feedback via \eqref{eq:preference-model} already suffices to learn the true distribution $\pi^*$.
Nonetheless, we posit that starting from a well-aligned $\piref$ can result in improved guarantees by mitigating the dependency of constants on $R$, which we leave to future work.

Next, we turn our focus to alignment methods that require an explicit reward model. As per our philosophy, we emphasize that the methods are derived from a distribution learning perspective rather than reward maximization.

\section{Preference Distillation Approach}\label{sec:pref}

Since the popularization of RLHF, the use of reward modeling has become popular in the research community and resulted in various extensions~\citep{christiano17deep}.
While the main role of the reward model in the RLHF objective~\eqref{eq:rlhf} is to view alignment as an RL problem, recent studies have attempted to use the reward model for supervised learning losses, i.e., objectives that do not require RL to solve~\cite{guo24direct,fisch25robust}.
These efforts can be seen as \emph{distilling information} from the reward model as pointed out by~\citet{fisch25robust}.
The main benefit of these methods is that they can avoid RL algorithms, which are typically slow to converge.
While reward model training is an extra burden to perform compared to purely likelihood-based methods such as DPO or our PMLE, the compute cost for doing so is typically quite low because it usually suffices to train a shallow network on top of an existing LLM's frozen torso.

\paragraph{Reward model.}

Due to our preference model~\eqref{eq:preference-model}, learning a reward model $R:\cX\times\cA\to\RR$ is equivalent to learning a language model $\pi$ and then setting $R(x,a) = \gam \ln\pi(a\mid x)$ up to an additive constant.
Conversely, given a reward model $R(x,a)$, we can estimate an LM by
\begin{align}\label{eq:policy_from_reward_model}
  \pi(a\mid x) \propto \exp(\gam^{-1}R(x,a)), \quad \forall x\in \cX.
\end{align}
Note that this is a model from which sampling is computationally intractable in general.
Formally, we assume that we are given a reward model class $\cR$ of rewards $R:\cX \times \cA \rarrow \RR$ and learn:
\begin{align}\label{eq:reward-model-training}
  \hR = \argmin_{R\in \cR}~\sum_{D_n} -\ln\sig(R(x,a^+) - R(x,a^-)).
\end{align}
This is equivalent to the PMLE objective under \eqref{eq:policy_from_reward_model} but with the constraint $R\in\cR$.

\paragraph{Preference distillation.}
One popular method for distilling rewards is the REBEL algorithm \citep{gao24rebel}.
Motivated by the characterization of the RLHF solution under the all-policy assumption~\cite{rafailov23direct}, REBEL aims to extract information from relative reward values of paired responses, enforcing
$
    \ln \frac{\pi(a^+\mid x) / \piref(a^+ \mid x)}{\pi(a^-\mid x) / \piref(a^-\mid x)} \approx \eta(\hR(x,a^+) - \hR(x,a^-))
$
by optimizing a squared loss
\begin{align}\label{eq:rebel}
    &\sum_{D_n} \Big(\ln \frac{\pi(a^+| x) \piref(a^-| x)}{\pi(a^-| x) \piref(a^+| x)} - \eta(\hR(x,a^+) - \hR(x,a^-))\Big)^2
\end{align}
where $\eta > 0$ controls the strength of the reward signals. In our assumption, the reward model can be seen as a shifted version of $\gam \ln\pi^*(a\mid x)$, so we could optimize \eqref{eq:rebel} without the $\piref$ terms, replacing $\eta$ by $\gam^{-1}$.
However, the use of squared loss in \eqref{eq:rebel} is not well justified from a statistical perspective, and it is unclear if squared loss should be preferred over any other loss, e.g., absolute loss.

What is then the appropriate error measure?
Our framework tells us that learning a reward model amounts to learning a preference model. In other words, we have trained a \emph{preference simulator}: a non-generative language model estimate $\wtpi(a\mid x) \propto \exp(\gam^{-1}\hR(x,a))$ from which preference can be sampled for any pair of responses as $y \sim \Bernoulli(\PP_{\wtpi}(a \succ b \mid x))$. Plugging this into PMLE would yield a natural distribution learning objective.
However, this process introduces additional randomness which can hinder optimization.
Instead, observe that we can evaluate the expectation of the PMLE objective and replace the discrete label $y$ with the \emph{expected} preference
\begin{align*}
    \PP_{\wtpi}(a^+ \succ a^- \mid x) & = \fr{\wtpi(a^+ \mid x)^\gam}{\wtpi(a^+ \mid x)^\gam + \wtpi(a^- \mid x)^\gam} = \sig( \hR(x,a^+) - \hR(x,a^-)).
\end{align*}
Then, minimizing the log-loss with respect to this synthetic preference is equivalent to minimizing the KL between the binary preference distributions $\mathrm{Bern}(\PP_{\wtpi}(a^+ \succ a^- \mid x))$ and $\mathrm{Bern}(\PP_{\pi}(a^+ \succ a^- \mid x))$:
\begin{align}\label{eq:distill}
    \cL_{\distill}(\pi) := \fr{1}{n} \sum_{D_n} \KL\!\big(\mathrm{Bern}(\PP_{\wtpi}(a^+\succ a^-\mid x)) \,\Vert\, \mathrm{Bern}(\PP_\pi(a^+\succ a^-\mid x))\big) + \mathrm{const.}
\end{align}
As in the PMLE (Section \ref{sec:mle}), in practice we add a KL regularizer:
\begin{align}\label{eq:pref_distill}
    \cL_{\distill,\beta}(\pi) := \cL_{\distill}(\pi) + \beta \KL(\pi\Vert \piref).
\end{align}
We remark that the data for reward model training \eqref{eq:reward-model-training} and preference distillation \eqref{eq:distill} can come from different datasets; our theoretical analysis is easily adapted.

\paragraph{Convergence guarantee.}
The family of (non-generative) LMs induced by the reward model class $\cR$ is defined as
\begin{align*}
    \cP_\gam(\cR) := \big\{\pi: \pi(a\mid x) \propto \exp(\gam^{-1} R(x,a)), \forall a,x\in \cX \text{ for some } R\in\cR \big\}.
\end{align*}

\begin{assumption}\label{assp:reward_model}
The reward-induced LM class $\cP_\gam(\cR) \subseteq \Pi$.
\end{assumption}
This assumption is related to the generator-verifier gap, which informally states that verifying whether a given answer is correct or not is easier than generating a correct answer~\cite{li24benchmarking, west24generative}.
Such a gap implies that $\cR$ is easier to learn than $\Pi$ from a learning-theoretic perspective ($|\cR| \ll |\Pi|$), and is speculated to hold for LLMs in practice~\cite{swamy25all}.
Assumption~\ref{assp:reward_model} can also be justified by the fact that the RM is often built on top of the supervised fined-tuned model's (frozen) torso.
Denoting by $C_\cR := C_{\cP_\gam(\cR)}$ the generalized coverage coefficient of the induced subclass, under Assumption \ref{assp:reward_model} it holds that $C_\cR \le C_\Pi$.
\begin{theorem}
\label{thm:pref_distill}
The pref. distill. estimate $\hpi = \argmin\limits_{\pi\in\Pi} \cL_{\distill}(\pi)$ satisfies with probability at least $1-\delta$,
\begin{align}\label{eq:distill_guarantee}
\EE_{x\sim\cD}\left[\KL(\pi^*(x)\Vert \hpi(x))\right] \lesssim \frac{C_{\Pi}}{\gam^2} \cd \frac{\ln(|\Pi|/\delta)}{n}.
\end{align}
\end{theorem}
See Appendix \ref{app:proof_distill} for the proof.

\paragraph{Benefits of distillation.} The above rate $n\gtrsim C_\Pi\ln|\Pi|$ is equal to that of PMLE \eqref{eq:pmle_guarantee} since we assume we learn $\hpi\in\Pi$ using responses $D_n$ generated from $\mu$, same as the reward model. However, if we have access to a `stronger' base model $\pi_0$, it is natural to learn $\hpi$ using responses from $\pi_0$ instead. In this scenario, we instead obtain a rate of $\gam^{-2}(C_0\ln|\Pi| + C_\cR\ln|\cR|)$, where $C_0$ is a slightly modified coverage coefficient for $\pi_0$ instead of $\mu$. This improves over the PMLE rate if $\pi_0$ is comparatively well-aligned so that $C_0 < C_\Pi$, thus rigorously establishing the benefits of distillation with a strong base model.
See Theorem \ref{thm:banana} in Appendix \ref{app:proof_distill} for a formal statement and proof.

\section{Reverse KL Minimization Approach}\label{sec:revkl}

Our two proposed methods both maximize a preference likelihood and ultimately enjoy a guarantee on the forward KL divergence $\EE_x[\KL(\pi^*(x) \Vert \hpi(x))]$.
However, it is also plausible to aim to minimize the reverse KL divergence $\EE_x[\KL(\hpi(x) \Vert \pi^*(x))]$ to learn the distribution $\pi^*$.
Reverse KL has the well-known `mode seeking' behavior as opposed to `mode covering' behavior of the forward KL.
This mode-seeking behavior tends to find distributions that generate realistic content and has been preferred in image generative models~\cite{goodfellow14generative,mao19mode}.

In this section, we explore the reverse KL formulation for alignment under our  assumption~\eqref{eq:preference-model}, which turns out to be a generalization of the original RLHF framework \eqref{eq:rlhf} \citep{stiennon20learning,ouyang22training}.
Directly minimizing the reverse KL w.r.t. the target LM $\pi^*$ would yield:
\begin{align}\label{eq:rkl_ideal}
    \hat{\pi} = \argmin_{\pi\in\Pi} \mathbb{E}_{x \sim \mathcal{D}} \big[\mathbb{E}_{a \sim \pi(x)} [-\ln \pi^*(a \mid x)] - H(\pi(x))\big],
\end{align}
where $H(\pi(x))$ is the Shannon entropy of $\pi(x)$.
However, this requires rewards of the form $-\ln\pi^*$, which is the very object we are trying to estimate. To solve this issue, we propose to find a plugin estimator from a surrogate class of language models that are easier to train but harder to sample from.
Specifically, we determine $\wtpi = \argmin_{\pi\in \cP_\gam(\cR)} \cL_\pmle(\pi)$ (with a suitable regularization), which is equivalent to obtaining $\hR$ via~\eqref{eq:reward-model-training} followed by setting $\wtpi(a\mid x) \propto \exp(\gam^{-1}\hR(x,a))$ as before.
Then we can plug in our learned $\wtpi$ to $\pi^*$ in \eqref{eq:rkl_ideal} to arrive at the objective
\begin{align*}
\argmin_{\pi\in\Pi} \mathbb{E}_{x \sim \mathcal{D}}\left[\mathbb{E}_{a \sim \pi(x)}\big[-\gam^{-1} \hR(x,a)\big] - H(\pi(x))\right].
\end{align*}
The normalization constant, which is prohibitive to compute in practice, naturally disappears as we only require relative rewards for optimization. Lastly, we again add a KL regularizer w.r.t. $\piref$:
\begin{align}\label{eq:rkl}
    \cL_{\rkl,\beta}(\pi) :&= \frac{1}{n} \sum\limits_{(x,\cd,\cd) \in D_n} -\EE_{a \sim \pi(x)}\big[\hR(x,a)\big] - \gamma H(\pi(x)) + \beta \KL(\pi(x) \Vert \piref(x))
\end{align}
where $\beta$ and $\gamma$ control
the relative weights of the policy entropy and KL regularizer, respectively.
In practice, as in standard RLHF pipelines \citep{ouyang22training,bai22training}, one first fits a reward model $\hR$ to approximate the underlying true reward $R^*$ from pairwise preferences, then applies an RL algorithm (e.g., PPO~\citep{schulman2017proximal}) to minimize the objective \eqref{eq:rkl}.

\paragraph{Relation to RLHF.}
At face value, RKL can be seen as a generalization of RLHF since the RKL objective with $\gam=0$ is equal to the RLHF objective.
Conversely, under our preference assumption~\eqref{eq:preference-model}, RLHF itself can be interpreted as minimizing a reverse KL in the limit $\gam\rarrow0$.
Thus, our result on RKL can be viewed as providing \emph{theoretical justification} for the RLHF objective~\eqref{eq:rlhf} which has been widely viewed as the gold standard for alignment~\citep{stiennon20learning,bai22training,rafailov23direct}, \emph{while also providing a minor correction.}
Note that such a connection to RLHF may not be surprising given that the max entropy RL can be seen as reverse KL minimization~\cite{ziebart10modeling}. \looseness=-1

\paragraph{The dilemma of RLHF.}
As discussed in the introduction, RLHF suffers from the following asymptotic dichotomy, assuming $\mu$ has sufficient coverage:
\begin{enumerate}[leftmargin=6mm]
    \item \textbf{Underfitting} (fix $\beta>0$ for all $n$): The learned LM is a proper distribution but cannot be too far from the reference policy $\pi_0$.
    \item \textbf{Degenerate} (take $\beta = \beta_n \downarrow 0$ as $n\rarrow \infty$): The learned LM collapses to a degenerate solution.
\end{enumerate}
Neither is desirable.
This phenomenon can also be seen in our toy experiment shown in Figure \ref{fig:rlhf_comparison}.\footnote{Experimental details: we set both the vocabulary and context sizes to $10$ and instantiate $\pi^*,\piref$ as random categorical distributions.
For each method and $n$, we report the result from the best $\beta$ from a wide range.}
As demonstrated in Theorem \ref{thm:reverse_kl}, with our approach the forward KL indeed tends to zero as the sample size $n$ increases, whereas for RLHF, it eventually saturates at a non-zero value.

\paragraph{Convergence guarantee.} With the objective $\cL_{\rkl}:=\cL_{\rkl,0}$, we are indeed able to obtain the following guarantee proved in Appendix \ref{app:proof_reverse}:

\begin{theorem}
\label{thm:reverse_kl}
    The reverse KL estimate $\hpi = \argmin\limits_{\pi \in \Pi} \cL_{\rkl}(\pi)$ satisfies with probability at least $1 - \delta$,
    \begin{align}\label{eq:rev_guarantee}
        \EE_{x \sim \cD}\left[\KL(\pi^*(x) \Vert \hpi(x))\right] \lesssim \frac{\ln(|\Pi|/\delta)}{n} + \frac{C_{\cR}}{\gam^2}\cd \frac{\ln(|\cR|/\delta)}{n}.
    \end{align}
\end{theorem}

\paragraph{Why does reverse KL attain a better bound?}

The reverse KL formulation results in an improved upper bound for the \emph{forward} KL that depends on the coverage coefficient of $\cP_\gam(\cR)$ rather than $\Pi$ as in previous bounds \eqref{eq:pmle_guarantee}, \eqref{eq:distill_guarantee}. In particular, under Assumption \ref{assp:reward_model}, $\ln|\Pi|$ and $C_{\cR}\ln|\cR|$ may both be much smaller than $C_\Pi\ln|\Pi|$, or even the improved preference distillation guarantee $C_0\ln|\Pi|$ \eqref{eq:new_distill_guarantee}.

Astute readers may wonder: How can reverse KL avoid $C_\Pi$ (or $C_0$) while preference distillation does not, even though they both leverage the reward model? The reason is that the \textit{policy learning step} of preference distillation still relies on response pairs $(a^+,a^-)$ sampled from $\mu$ (or $\pi_0$), unlike reverse KL which uses responses for the lightweight \textit{reward modeling step} only. Moreover, bounding forward instead of reverse KL (even though the objective is derived from minimizing reverse KL) allows us to avoid comparing the coverage of $\hpi$ against $\mu$, which is not guaranteed by Definition \ref{def:cov}. Nevertheless, the forward and reverse KL error may still be compared (however suffering a constant exponential in $R$), as we show in Proposition \ref{prop:klsame}.

An alternative method is to directly optimize forward KL: $\argmin_{\pi\in\Pi} \sum_{(x,\cd,\cd)\sim D_n} \KL(\wtpi(x)\Vert \pi(x))$. However, this is computationally intractable since it requires evaluating normalization constants or sampling from unnormalized distributions. We defer detailed discussion to Appendix \ref{sec:addtional_remark}.

\begin{table}[t]
    \centering
    \caption{\textbf{Results on TL;DR dataset with Pythia 2.8B and 6.9B}. Win-rate is evaluated by GPT-4 and reward model (RM) score evaluated by the trained reward model. We report mean/std over three random seeds.}
    \label{tab:tldr}
    \resizebox{\linewidth}{!}{%
    \begin{tabular}{lcccccc}
    \toprule
    & \multicolumn{3}{c}{Pythia 2.8B} & \multicolumn{3}{c}{Pythia 6.9B} \\
    \cmidrule(lr){2-4}\cmidrule(lr){5-7}
    Algorithm & Win-rate$(\uparrow)$ & RM score$(\uparrow)$ & $\KL(\pi \Vert \piref)(\downarrow)$ & Win-rate$(\uparrow)$ & RM score$(\uparrow)$ & $\KL(\pi \Vert \piref)(\downarrow)$ \\
    \midrule
    DPO            & 47.3\tiny{$\pm$1.01} & \textbf{2.66}\tiny{$\pm$0.03} & 64.35\tiny{$\pm$0.68} & 55.6\tiny{$\pm$0.53} & 6.09\tiny{$\pm$0.02} & 54.45\tiny{$\pm$0.82} \\
    PMLE           & \textbf{49.1}\tiny{$\pm$1.43} & 2.38\tiny{$\pm$0.04} & \textbf{31.42}\tiny{$\pm$0.79} & \textbf{59.0}\tiny{$\pm$0.74} & \textbf{6.27}\tiny{$\pm$0.03} & \textbf{24.20}\tiny{$\pm$0.92} \\
    \cmidrule{1-7}\morecmidrules\cmidrule{1-7}
    REBEL          & 71.0\tiny{$\pm$1.26} & 2.86\tiny{$\pm$0.02} & \textbf{25.67}\tiny{$\pm$0.61} & 82.5\tiny{$\pm$0.93} & 7.64\tiny{$\pm$0.02} & 29.80\tiny{$\pm$0.71} \\
    Pref. Distill. & \textbf{73.8}\tiny{$\pm$1.47} & \textbf{2.91}\tiny{$\pm$0.03} & 28.38\tiny{$\pm$0.82} & \textbf{83.9}\tiny{$\pm$0.59} & \textbf{7.66}\tiny{$\pm$0.02} & \textbf{29.05}\tiny{$\pm$1.02} \\
    \cmidrule{1-7}\morecmidrules\cmidrule{1-7}
    RLHF           & 72.0\tiny{$\pm$0.67} & 2.95\tiny{$\pm$0.05} & \textbf{24.94}\tiny{$\pm$1.12} & 82.0\tiny{$\pm$0.92} & \textbf{7.85}\tiny{$\pm$0.03} & 24.38\tiny{$\pm$0.77} \\
    Reverse KL     & \textbf{73.0}\tiny{$\pm$0.93} & \textbf{3.03}\tiny{$\pm$0.04} & 25.91\tiny{$\pm$0.94} & \textbf{84.0}\tiny{$\pm$1.01} & 7.77\tiny{$\pm$0.01} & \textbf{24.12}\tiny{$\pm$0.63} \\
    \bottomrule
    \end{tabular}%
    }
    \vskip -10pt 
\end{table}

\begin{table}[t]
    \caption{\textbf{General chat results on UltraFeedback}. LLM-as-a-judge evaluations on MT-Bench, AlpacaEval 2.0 (length-controlled and raw win-rate, \%), and Arena Hard v0.1 (\%). Bold denotes the winner within each pair (DPO/PMLE and REBEL/Preference Distillation).}
    \vskip -5pt 
    \label{tab:academic}
    \centering
    \resizebox{0.75\linewidth}{!}{%
    \begin{tabular}{lcccc}
        \toprule
        \multirow{2}{*}[\dimexpr0.5\baselineskip\relax]{Model} & \multirow{2}{*}[\dimexpr0.5\baselineskip\relax]{MT-Bench} & \multicolumn{2}{c}{AlpacaEval 2.0} & \multirow{2}{*}[\dimexpr0.5\baselineskip\relax]{Arena Hard} \\
        & & LC win-rate & Win-rate & \\
        \midrule
        Base (LLaMA-3-8B-Instruct) & 8.10 & 29.62 & 28.71 & 26.8 \\
        \cmidrule{1-5}\morecmidrules\cmidrule{1-5}
        DPO            & \textbf{8.32} & 44.72 & 43.94 & 45.1 \\
        PMLE           & 8.30 & \textbf{53.23} & \textbf{50.43} & \textbf{46.5} \\
        \cmidrule{1-5}\morecmidrules\cmidrule{1-5}
        REBEL          & 8.17 & 52.04 & 44.99 & 39.8 \\
        Pref.~Distill. & \textbf{8.18} & \textbf{54.49} & \textbf{46.86} & \textbf{41.9} \\
        \bottomrule
    \end{tabular}
    }
    \vskip -10pt 
\end{table}

\section{Experiments}
\label{sec:exp_main}

In this section, we present empirical results showing that our framework yields competitive performance in practice. Specifically, we compare our proposed methods (PMLE, reverse KL, and preference distillation) against their well-established baselines (DPO, RLHF, and REBEL) on a range of language tasks.

\begin{highlight}
    \begin{quoting}[leftmargin=0.5em, rightmargin=0.5em]
    \textbf{Evaluation Point.} \textit{Can theoretically-justified objectives derived purely from our distribution-learning assumption~\eqref{eq:preference-model} remain competitive with strong empirical baselines (DPO, REBEL, RLHF) in practice?}
    \end{quoting}
\end{highlight}

\subsection{TL;DR Summarization}

We follow the standard TL;DR setup~\citep{stiennon20learning} of \citet{gao24rebel,song24theimportance} with Pythia-2.8B and Pythia-6.9B~\citep{biderman2023pythia}; full training and evaluation details are deferred to Appendix~\ref{app:tldr}. We report the GPT-4 win-rate against human references (over 600 test samples), the trained reward model score, and $\KL(\pi\Vert\piref)$.

\paragraph{Results.} See Table~\ref{tab:tldr} for the results. For both Pythia-2.8B and Pythia-6.9B, our distribution-learning objectives achieve higher win-rates than their respective baselines in this experiment, where win-rate serves as the most direct measure of LM quality.
Additionally, since PMLE implements a KL regularizer with online data, it achieves a much lower KL term compared to DPO, which solely relies on the offline dataset; this finding aligns with the results reported by \citet{song24theimportance}.
As for RLHF and REBEL, both methods use the same KL penalty for each experiment, naturally leading to similar $\KL(\pi \Vert \piref)$ values.
Overall, these experiments provide evidence that the algorithms derived from our assumption~\eqref{eq:preference-model-intro} can remain competitive with popular baselines.

\subsection{General Chat}

We train LLaMA-3-8B-Instruct~\citep{grattafiori2024llama} on UltraFeedback~\citep{cui2023ultrafeedback} using ArmoRM-Llama3-8B-v0.1~\citep{wang2024interpretable} as the reward model, following \citet{gao24rebel}; full setup is in Appendix~\ref{app:general}. We focus on the likelihood-based pair (PMLE vs.\ DPO) and the reward-distillation pair (Pref.\ Distill.\ vs.\ REBEL), since these differ algorithmically from their baselines, whereas RKL only adds an entropy term to the same PPO update as RLHF. Chat quality is evaluated with three LLM-as-a-judge benchmarks: MT-Bench~\citep{zheng2023judging}, AlpacaEval~2.0~\citep{dubois2024lengthcontrolled}, and Arena Hard~\citep{li2025crowdsourced}.

\paragraph{Results.} Our main goal is to examine whether objectives derived from the distribution-learning view remain competitive with strong empirical baselines in realistic chat settings. The results are encouraging: PMLE and preference distillation perform favorably against DPO and REBEL across the evaluated benchmarks, with particularly strong alignment-quality gains on AlpacaEval~2.0 and Arena Hard while remaining comparable on MT-Bench. Together with the TL;DR results, these findings suggest that our theoretically motivated objectives are not only competitive with strong baselines, but also promising for improving practical alignment performance.
\looseness=-1

\section{Conclusion}
\label{sec:conclusion}

We formulated alignment as distribution learning from a single explicit assumption~\eqref{eq:preference-model} on how preferences relate to the target LM, and derived three principled methods (PMLE, preference distillation, RKL) that correct and generalize existing approaches with $O(1/n)$ forward-KL guarantees, accompanied by empirical results showing competitive performance and gains in several TL;DR and chat-evaluation settings.
Future directions include comparing mode-seeking versus mode-covering objectives across domains, extending the framework to alternative divergences and preference models, and tightening the exponential dependence on $R$ via stronger assumptions on $\piref$.
\looseness=-1


\bibliographystyle{unsrtnat}
\bibliography{neurips_2026}


\newpage
\appendix
\addcontentsline{toc}{section}{Appendix}
\part{\LARGE Appendix}

\parttoc

\section{Related Work}
\label{app:related}

\paragraph{Preference optimization with RL.}

A widely adopted paradigm in preference optimization is Reinforcement Learning from Human Feedback (RLHF). In this framework, one first trains a reward model--effectively serving as a classifier--on a preference dataset collected from human annotators, and subsequently leverages the learned reward model to run RL algorithms such as PPO \citep{christiano17deep,ziegler2019fine}. RLHF and its variants have been instrumental in training prominent LLMs such as ChatGPT \citep{openai2022chatgpt}, and have achieved remarkable success across diverse applications such as text summarization, question answering, instruction following, and text-to-image generation \citep{stiennon20learning,nakano2022webgpt,ouyang22training,lee2023aligning,liang2024rich}. We point the interested reader to \citet{kaufmann2024survey} for a recent dedicated survey on RLHF.

\paragraph{Without RL and without a reward model.}

Direct Preference Optimization (DPO) dispenses with an explicit reward model by treating the log-ratio of each preference pair as a training signal and directly training the policy with a single contrastive cross-entropy loss \citep{rafailov23direct}. Such an RL-free objective was shown to match PPO-based RLHF without requiring a reward model, value network, or on-policy sampling, and has led to variants such as distilled DPO \cite{tunstall2024zephyr}, Cal-DPO \cite{Cal-DPO2024}, diffusion DPO \cite{wallace2023diffusion}, $\Psi$PO~\cite{azar24general}, SLiC/SLiC-HF \cite{zhao23calibrating,zhao23slic-hf}, GPO~\cite{tang24generalized}, $\chi$PO~\cite{huang25correcting}, R-DPO~\cite{park24disentangling}, ODPO \cite{amini24direct}, SimPO \cite{meng24simpo}, RRHF \cite{yuan23rrhf}, KTO \cite{ethayarajh25model}, ORPO \cite{hong24orpo}, and many more.

At the same time, such direct optimization from preference labels has been noted to underperform along some dimensions compared to conventional RLHF.
One challenge stems from relying exclusively on an offline dataset, which can induce out-of-distribution responses.
This is likely due to insufficient on-policy interaction during training \citep{song24theimportance}. Some hybrid approaches have been proposed to overcome this issue: iterative DPO performs iterative training with labeled online preferences \citep{liu2024iterative}, HyPO combines offline data for preference optimization and online data for KL regularization \citep{song24theimportance}, and online DPO utilizes fast and slow chasing to simulate competition \citep{qi2024onlinedpo}.

\paragraph{Without RL but with a reward model.}
Another prominent method of preference optimization is reward distillation. This line of work aims to distill information on a reward model's preferences directly into the policy.
As discussed in Section \ref{sec:pref}, the REBEL objective \citep{gao24rebel} regresses the log-ratio of the likelihoods of two responses on the reward difference using a simple squared-loss objective, which is repeated with batches of on-policy responses.
Reward distillation from \citet{fisch25robust} can be seen as a simplified version of REBEL where we only use the responses from the preference dataset.
DRDO learns a reward model and policy in one pass by jointly matching oracle rewards while also learning human preferences \citep{nath2025simul}.
Finally, \citet{zhang2025distilldatarewardssmaller} develops an LLM distillation pipeline to distill both data and rewards.

\paragraph{Theoretical analyses of preference optimization.}
\citet{zhan2024provable} studies offline preference-based RL with an MLE-based reward model similar to ours, but only obtains guarantees in terms of maximizing the policy value. \citet{xie2025exploratory} proposes an exploratory version of DPO which is shown to achieves $\widetilde{O}(\sqrt{T})$ regret with a favorable coverage parameter. \citet{zhang2024self} proposes an online direct alignment algorithm which also attains $\widetilde{O}(\sqrt{T})$ regret. \citet{xiong24iterative} derives regret bounds for online and offline versions of RLHF under a linearly parametrized reward model; see also \citet{foster2025good} for a theoretical analysis of RL with linear-softmax policies. \citet{cen2025valueincentivized} introduces VPO, a value-regularized DPO-type objective for both online and offline RLHF, and also prove regret bounds under linear rewards. The $\chi$PO algorithm is shown to attain optimal sample complexity, also in terms of regret, under a weaker single-policy concentrability \citep{huang25correcting}.

The work of \citet{agarwal25design} is most relevant to our paper, especially PMLE (Section \ref{sec:mle}): they develop a theoretical analysis of offline RLHF variants that minimize DPO-type objectives, and show a forward KL bound w.r.t. an optimal policy $\pi^*$. However, this formulation is not due to a distribution learning viewpoint but merely a byproduct of their strong realizability assumption (Assumption 3.2). Moreover, their upper bound has a square-root dependence on the excess risk $\epsilon = L(\pi)-L(\pi^*)$, which when applied to our framework yields a statistical rate of $1/\sqrt{n}$. In contrast, we obtain an improved rate of $1/n$ with a more careful analysis in Appendix \ref{sec:appendix_proofs}.

\section{Additional Remarks on Reverse KL}\label{sec:addtional_remark}

\paragraph{Prior smoothing.} A key distinction from the standard RLHF objective lies in how our formulation balances reward maximization with \emph{prior smoothing}. For an explicit comparison, we illustrate the effect of the additional entropy term for a toy alignment problem. Consider learning a $K$-categorical distribution on the simplex $\Delta_K = \{\vecp\in\RR_{\ge 0}^d : \sum_{k=1}^K p_k=1\}$, which can be viewed as a contextless language model with response length one and a vocabulary size of $K$. Suppose we are given a fixed vector $\vecp_0\in\Delta_K$ as the reference model and a learned reward function $\hat{\vecr} = (r_1,\cdots,r_K)$. In the standard RLHF approach \eqref{eq:rlhf} with temperature $\beta+\gam$, the optimal policy is given for all $k\in [K]$ by $\hp_k^{\rlhf} \propto p_{0,k} \exp(\frac{r_k}{\beta+\gamma})$. In contrast, our reverse KL objective \eqref{eq:rkl} can be rearranged as
\begin{align*}
\cL_{\rkl,\beta}(\vecp) &= - \vecp\cdot\hat{\vecr} - \gamma H(\vecp) + \beta \KL(\vecp||\vecp_0) \\
&= - \vecp\cdot\hat{\vecr} + (\beta+\gamma) \KL(\vecp||\vecp_0^\alpha) +\text{const.}
\end{align*}
where $\alpha:=\frac{\beta}{\beta+\gamma}$, resulting in the policy $\hp_k \propto p_{0,k}^\alpha \exp(\frac{r_k}{\beta+\gamma})$. The additional exponent $\alpha\in (0,1)$ acts to smooth the prior from $\vecp_0$ to $\vecp_0^\alpha$, allocating relatively more mass to actions with low initial probability. This boosts exploration especially for actions which were unlikely under the base policy, so that the estimated policy $\hat{\vecp}$ would not be too close to a degenerate distribution even if $\vecp_0$ is.

\paragraph{Intractability of forward KL.} An alternative is to directly optimize the forward KL:
\begin{align*}
\textstyle \argmin_{\pi\in\Pi} \sum_{(x,\cd,\cd)\sim D_n} \KL(\wtpi(x)\Vert \pi(x)).
\end{align*}
Here, we are not using $(a^+,a^-)$, so the dependence on $\mu$ disappears and we will not pay for $C_\Pi$, similarly to Theorem \ref{thm:reverse_kl}. However, how do we compute the forward KL? Direct computation is untenable due to the sheer size of $\cA$ in language models.
Instead, one may attempt to sample from $\wtpi(x)$ and perform stochastic optimization; however, such a sampling is not feasible because we only have access to the unnormalized version $\exp(\gam^{-1}\hR(x,\cd))$.
Another attempt is to use the fact that
\begin{align*}
    \KL(\wtpi(x)\Vert \pi(x)) = \EE_{a \sim \pi(x)} \bigg[\frac{\wtpi(a \mid x)}{\pi(a \mid x)} \ln\fr{\wtpi(a \mid x)}{\pi(a \mid x)}\bigg].
\end{align*}
While we do not have to sample from $\wtpi(x)$, we now have to evaluate the value of $\wtpi(a \mid x)$, which, again, is intractable.

\section{Theoretical Guarantees}
\label{sec:appendix_proofs}

\subsection{Auxiliary Lemmas}

We require the following basic results.

\begin{lemma}\label{lem:sig_diff}
For all $a,b \in \RR$ it holds that $|\sig(a) - \sig(b)| \ge \frac{1}{4} e^{-(|a|\vee |b|)} |a - b|$.
\end{lemma}

\begin{proof}
Recall that $\sigma(z) = 1/(1+\exp(-z))$ is the logistic sigmoid.
$\sig'$ is symmetric, so that
\begin{equation*}
\sig'(z) = \sig'(|z|) = \fr{1}{1+e^{|z|}} \fr{1}{1+e^{-|z|}} \ge \fr{1}{2(1+e^{|z|})} \ge \fr14 e^{-|z|}.
\end{equation*}

It suffices to assume $b > a$ due to symmetry. Then,
  \begin{align*}
    \sig(b) - \sig(a)
    = \int_{a}^b \sig'(z) \dif z \ge \int_{a}^b \fr12 e^{-|z|} \dif z \ge \fr14 e^{-(|a|\vee|b|)} (b-a)
  \end{align*}
as desired.
\end{proof}

The next two lemmas will allow us to convert between the expectation of the log ratio (i.e., KL divergence) and squared log ratio. Let us define the auxiliary function
\begin{align*}
\psi(z):=\frac{z-1-\ln z}{(\ln z)^2}.
\end{align*}

\begin{lemma}\label{lem:kl_bounded_by_log_sq}
For $\rmax>1$, it holds for all $r \in \lparen 0,\rmax\rbrack$ that
  \begin{align*}
    r - 1 - \ln r \le (\fr12\vee\psi(\rmax)) (\ln r)^2\le \frac{\rmax}{(\ln \rmax)^2} (\ln r)^2.
  \end{align*}
\end{lemma}

\begin{proof}
Define the auxiliary function
\begin{align*}
f(r) := \fr12 (\ln r)^2 - (r-1-\ln r).
\end{align*}
For $r\in(0,1)$, it holds that $f(1) = 0$ and $f'(r) = \frac{\ln r - r + 1}{r} < 0$.
Thus, $f(r)>0$ which implies
\begin{equation*}
    r-1-\ln r\le\fr12 (\ln r)^2.
\end{equation*}
For $r\in [1,\rmax]$, it is easily checked that $\psi$ is nondecreasing on $(1,\infty)$ and thus
\begin{align*}
\fr{r - 1 - \ln r}{(\ln r)^2} =\psi(r) \le \psi(\rmax) \le \frac{\rmax}{(\ln \rmax)^2},
\end{align*}
as was to be shown.
\end{proof}

\begin{lemma}\label{lem:opposite}
For $\rmin>0$, it holds for all $r\in[\rmin,\infty)$ that
\begin{align*}
r-1-\ln r \ge \frac{1}{e(\ln \rmin^{-1}\vee 1)} (\ln r)^2.
\end{align*}
\end{lemma}

\begin{proof}
The function $\psi$ defined in Lemma \ref{lem:kl_bounded_by_log_sq} extends to a nondecreasing continuous function on $(0,\infty)$ by setting $\psi(1):=\fr12$. When $r\ge e^{-1}$, it follows that $\psi(r)\ge\psi(e^{-1})=e^{-1}$.

When $\rmin\le r< e^{-1}$, we use the fact that $\ln r\le \frac{1}{1-e^{-1}}(r-1)$ to bound
\begin{align*}
\psi(r) \ge \frac{(1-e^{-1})\ln r - \ln r}{(\ln r)^2} = \frac{1}{e\ln r^{-1}} \ge \frac{1}{e\ln \rmin^{-1}}.
\end{align*}
\end{proof}

\begin{lemma}[Symmetrization inequality]\label{thm:symmetrization}
Let $D_n,\tD_n$ be two datasets of $n$ i.i.d. samples, $C(\pi,D_n)$ be any functional of a policy $\pi$ and dataset $D_n$, and $\hpi := \hpi(D_n)$ be any estimator computed from $D_n$. Then with probability $1-\delta$, it holds that
\begin{align*}
-\log\EE_{\tD_n}[\exp(C(\hpi,\tD_n))] \le -C(\hpi,D_n) + \ln(|\Pi|/\delta).
\end{align*}
\end{lemma}

\begin{proof}
This is shown for example in the proof of Theorem 6 in \citet{foster21efficient}.
\end{proof}

\subsection{Proofs for Section~\ref{sec:mle}}\label{app:proof_mle}

The following convergence bound for maximum likelihood estimators is mostly classical \citep{tong07entropy,van2009empirical}; for completeness, we provide a brief proof following Theorem 6 of \citet{foster21efficient}.

\begin{proposition}\label{thm:oracle}
Let $\hpi = \argmin_{\pi\in\Pi} \cL_{\pmle}(\pi)$ with $\beta = 0$.
Then, with probability at least $1-\delta$,
    \begin{align*}
\EE_{x\sim \mathcal{D},a,b\sim\mu(x)}\left[(\PP_{\widehat\pi}(a \succ b \mid x) - \PP_*(a \succ b \mid x))^2\right] \le \frac{4\ln(|\Pi|/\delta)}{n}.
    \end{align*}
\end{proposition}

\begin{proof}
Recall that each preference pair $(x,a^+,a^-)$ is collected by first sampling $a,b$ independently from $\mu(x)$ and setting $(a^+,a^-) = (a,b)$ with probability $\PP_*(a\succ b\mid x)$. In other words, for the indicator $y = 1_{\{a^+=a\}}$ such that $\PP(y=1) = \PP_*(a \succ b \mid x)$, we can write
\begin{align*}
\cL_{\pmle}(\pi) = \frac{1}{n} \sum_{(x,a,b)\in D_n} -y \ln\PP_\pi(a \succ b \mid x) - (1-y) \ln \PP_\pi(b \succ a \mid x),
\end{align*}
where we have abused notation to write the sum over $(x,a,b)$ corresponding to each $(x,a^+,a^-)$ as a sum over $(x,a,b)\in D_n$. Define the quantity
\begin{align*}
C(\pi,D_n)& = \frac{1}{2} \sum_{(x,a,b)\in D_n} y\ln \frac{\PP_\pi(a \succ b \mid x)}{\PP_*(a \succ b \mid x)} + (1-y)\ln \frac{\PP_\pi(b \succ a \mid x)}{\PP_*(b \succ a \mid x)}\\
&= \frac{n}{2}(\cL_{\pmle}(\pi^*) - \cL_{\pmle}(\pi))
\end{align*}
and $\hpi$ as the minimizer of $\cL_{\pmle}(\pi)$ for $\pi\in\Pi$. It follows from Lemma~\ref{thm:symmetrization} that
\begin{align*}
-\log\EE_{\tD_n}[\exp(C(\hpi,\tD_n))] \le -C(\hpi,D_n) + \ln(|\Pi|/\delta) \le \ln(|\Pi|/\delta)
\end{align*}
and
\begin{align*}
&-\log\EE_{\tD_n}[\exp(C(\hpi,\tD_n))]\\
&= -n\log\EE_{x\sim\cD,a,b\sim\mu(x) } \EE_{y|a,b,x} \left[\left(\frac{\PP_{\hpi}(a \succ b \mid x)}{\PP_*(a \succ b \mid x)}\right)^{y/2} \left(\frac{\PP_{\hpi}(b \succ a \mid x)}{\PP_*(b \succ a \mid x)} \right)^{(1-y)/2}\right] \\
&= -n\log\EE_{x\sim\cD,a,b\sim\mu(x)} \left[ \sqrt{\PP_{\hpi}(a \succ b \mid x)\PP_*(a \succ b \mid x)} + \sqrt{\PP_{\hpi}(b \succ a \mid x)\PP_*(b \succ a \mid x)} \right].
\end{align*}
Writing $p_{\hpi} = \PP_{\hpi}(a \succ b \mid x)$ and $p_* = \PP_*(a \succ b \mid x)$ for simplicity, we further have
\begin{align*}
-\log\EE\left[\sqrt{p_{\hpi} p_*} + \sqrt{(1-p_{\hpi})(1-p_*)}\right]
&\ge 1 - \EE\left[\sqrt{p_{\hpi} p_*} + \sqrt{(1-p_{\hpi})(1-p_*)}\right] \\
&= \EE\left[\frac{1}{2}(\sqrt{p_{\hpi}} - \sqrt{p_*})^2 + \frac{1}{2}(\sqrt{1-p_{\hpi}} - \sqrt{1-p_*})^2\right] \\
&= \EE\left[\frac{(p_{\hpi}-p_*)^2}{2(\sqrt{p_{\hpi}} + \sqrt{p_*})^2} + \frac{(p_{\hpi}-p_*)^2}{2(\sqrt{1-p_{\hpi}} + \sqrt{1-p_*})^2}\right] \\
&\ge \fr14 \EE\left[(p_{\hpi}-p_*)^2 \right],
\end{align*}
which yields the desired bound.
\end{proof}

\textit{Proof of Theorem~\ref{thm:mle}.}
Our proof is partly inspired by~\citet[proof of Theorem 3.6]{agarwal25design}.
The key difference is that their theorem relies on an assumption that the \textit{population} loss of $\hpi$ is not too far away from that of $\pi^*$, which is rather strong.
In contrast, our theorem provides an end-to-end guarantee.
Furthermore, naively applying their theorem would result in an $1/\sqrt{n}$ rate rather than $1/n$.
We obtain an improvement by applying Schulman's trick~\cite{schulman20kl} followed by Lemma~\ref{lem:kl_bounded_by_log_sq}.
We elaborate more on this later in Remark~\ref{rem:novelty}.

Using Lemma \ref{lem:sig_diff} with the fact
\begin{align*}
\abs{\gam\ln\frac{\pi(a \mid x)}{\pi(b\mid x)}} = \abs{\bar{R}(x,a) - \bar{R}(x,b)} \le 2\gam R~,
\end{align*}
we can lower bound
\begin{align*}
&\EE_{x\sim \mathcal{D},a,b\sim\mu(x)}\left[(\PP_{\hpi}(a \succ b \mid x) - \PP_*(a \succ b \mid x))^2\right]\\
&\ge \frac{e^{-4\gam R}}{4} \EE_{x\sim \mathcal{D},a,b\sim\mu(x)} \bigg[\bigg(\gam \ln\frac{\hpi(a \mid x)}{\hpi(b\mid x)} - \gam\ln\frac{\pi^*(a \mid x)}{\pi^*(b\mid x)}\bigg)^2\bigg] \\
&= \frac{e^{-4\gam R}}{4} \EE_{x\sim \mathcal{D},a,b\sim\mu(x)} \left[(\Delta\bar{R}_{\hpi}(x,a) - \Delta\bar{R}_{\hpi}(x,b))^2\right] \\
&= \frac{e^{-4\gam R}}{2} \EE_{x\sim \mathcal{D},a\sim\mu(x)} \left[\Delta\bar{R}_{\hpi}(x,a)^2\right]
\end{align*}
where the last inequality is due to $\EE[(X - Y)^2] = 2\EE[(X- \EE[X])^2]$ when $X$ and $Y$ are i.i.d.
Thus, using Proposition~\ref{thm:oracle}, the difference in centered reward satisfies
\begin{align}\label{eq:delta}
\EE_{x\sim \mathcal{D},a\sim\mu(x)} \left[\Delta\bar{R}_{\hpi}(x,a)^2\right] \le
8e^{4\gam R} \cd \frac{\ln(|\Pi|/\delta)}{n}.
\end{align}
Define the normalizing factor
\begin{equation*}
Z_\pi(x):=\sum_{a\in\cA} \exp\left(\frac{1}{\gam}\bar{R}_\pi(x,a)\right) = \exp\left(-\frac{1}{\gam}\EE_{a\sim\mu(x)}[R_\pi(x,a)\mid x]\right), \quad Z_*:=Z_{\pi^*}
\end{equation*}
so that $\pi(a\mid x) = Z_\pi(x)^{-1} \exp(\gam^{-1}\bar{R}_\pi(x,a))$.
Due to Assumption \ref{assp:bdd}, for all $\pi\in\Pi,x\in\cX$ it holds that $|\cA|e^{-R}\le Z_\pi(x)\le |\cA|e^R$, so that
\begin{align}\label{eq:prob_ratio}
0< \frac{\hpi(a\mid x)}{\pi^*(a\mid x)} = \frac{Z_*(x)}{Z_{\hpi}(x)} \exp\left(\frac{1}{\gam}\Delta\bar{R}_{\hpi}(x,a)\right) \le e^{4R}.
\end{align}
Then, we bound the KL divergence between $\pi^*,\hpi$ using Schulman's trick~\cite{schulman20kl} followed by Lemma~\ref{lem:kl_bounded_by_log_sq}:
\begin{align}\label{eq:novelty}
\EE_{x\sim\cD}\left[\KL(\pi^*(x)\Vert \hpi(x))\right] & = \EE_{x\sim\cD, a\sim\pi^*(x)} \bigg[ \ln\frac{\pi^*(a\mid x)}{\hpi(a\mid x)} \bigg] \notag \\
&= \EE_{x\sim\cD, a\sim\pi^*(x)} \bigg[ \frac{\hpi(a\mid x)}{\pi^*(a\mid x)} - 1 -\ln \frac{\hpi(a\mid x)}{\pi^*(a\mid x)} \bigg] \notag\\
&\le (\fr12\vee\psi(e^{4R})) \EE_{x\sim\cD, a\sim\pi^*(x)} \bigg[\bigg(\ln \frac{\hpi(a\mid x)}{\pi^*(a\mid x)}\bigg)^2 \bigg].
\end{align}
Extracting the normalization constants, we further have that
\begin{align*}
&\EE_{x\sim\cD, a\sim\pi^*(x)} \bigg[\bigg(\ln \frac{\hpi(a\mid x)}{\pi^*(a\mid x)}\bigg)^2 \bigg]\\
&\le \EE_{x\sim\cD, a\sim\pi^*(x)} \bigg[ 2\left(\ln \frac{\hpi(a\mid x) Z_{\hpi}(x)}{\pi^*(a\mid x) Z_*(x)}\right)^2 + 2\left(\ln\frac{Z_*(x)}{Z_{\hpi}(x)}\right)^2 \bigg] \\
&= \frac{2}{\gam^2} \EE_{x\sim \mathcal{D},a\sim\pi^*(x)} \left[\Delta\bar{R}_{\hpi}(x,a)^2\right] + 2\EE_{x\sim\cD} \bigg[\left(\ln\frac{Z_*(x)}{Z_{\hpi}(x)}\right)^2 \bigg].
\end{align*}
Using Definition \ref{def:cov} and \eqref{eq:delta}, the first term is bounded as
\begin{align*}
\frac{2}{\gam^2} \EE_{x\sim \mathcal{D},a\sim\pi^*(x)} \left[\Delta\bar{R}_{\hpi}(x,a)^2\right] \le \frac{2C_\Pi}{\gam^2} \EE_{x\sim \mathcal{D},a\sim\mu(x)} \left[\Delta\bar{R}_{\hpi}(x,a)^2\right] \le
\frac{16C_\Pi e^{4\gam R}}{\gam^2} \cd \frac{\ln(|\Pi|/\delta)}{n}.
\end{align*}
For the second term, we first characterize an upper and lower bound on $\ln\frac{Z_{\hpi}(x)}{Z_*(x)}$.
Using
\begin{align*}
1 = \EE_{a\sim\pi^*(x)}\bigg[\frac{\hpi(a\mid x)}{\pi^*(a\mid x)}\bigg] = \frac{Z_*(x)}{Z_{\hpi}(x)} \EE_{a\sim\pi^*(x)}\bigg[\exp\left(\frac{1}{\gam}\Delta\bar{R}_{\hpi}(x,a)\right)\bigg],
\end{align*}
we have
\begin{align*}
\ln\frac{Z_{\hpi}(x)}{Z_*(x)} = \ln \EE_{a\sim\pi^*(x)}\bigg[\exp\left(\frac{1}{\gam}\Delta\bar{R}_{\hpi}(x,a)\right)\bigg]
   &\ge \fr1\gam \EE_{a\sim\pi^*(x)}[ \Dt \bar{R}_{\hpi}  (x,a)]
\end{align*}
where the last inequality is by Jensen's inequality.
Moreover, using the inequality $e^x\le 1+x+\frac{e^A}{2}x^2$ valid for all $x\in(-\infty,A]$, we have
\begin{align*}
\ln\frac{Z_{\hpi}(x)}{Z_*(x)}
&= \ln \EE_{a\sim\pi^*(x)}\bigg[\exp\left(\frac{1}{\gam}\Delta\bar{R}_{\hpi}(x,a)\right)\bigg]
\\&\le \EE_{a\sim\pi^*(x)}\bigg[\exp\left(\frac{1}{\gam}\Delta\bar{R}_{\hpi}(x,a)\right)\bigg] - 1
\\&\le \frac{1}{\gam}\EE_{a\sim\pi^*(x)} \left[\Delta\bar{R}_{\hpi}(x,a)\right] + \frac{e^{2R}}{2\gam^2} \EE_{a\sim\pi^*(x)} \left[\Delta\bar{R}_{\hpi}(x,a)^2\right].
\end{align*}
Thus, we have
\begin{align*}
  \abs{\ln \fr{Z_{\hpi}(x)}{Z_*(x)} }
  \le \abs{\fr1\gam \EE_{a\sim\pi^*(x)}[ \Dt \bar{R}_{\hpi}  (x,a)]} +  \frac{e^{2R}}{2\gam^2} \EE_{a\sim\pi^*(x)} \left[\Delta\bar{R}_{\hpi}(x,a)^2\right],
\end{align*}
which implies, using $\forall x,y\in \RR, (x+y)^2 \le 2 x^2 + 2y^2$,
\begin{align*}
&\EE_{x\sim\cD} \bigg[\left(\ln\frac{Z_*(x)}{Z_{\hpi}(x)}\right)^2 \bigg]\\
&\le \frac{2}{\gam^2}\EE_{x\sim\cD}\bigg[\left(\EE_{a\sim\pi^*(x)}\left[\Delta\bar{R}_{\hpi}(x,a)\right]\right)^2\bigg] + \frac{e^{4R}}{2\gam^4}\EE_{x\sim\cD}\bigg[\left(\EE_{a\sim\pi^*(x)} \left[\Delta\bar{R}_{\hpi}(x,a)^2\right]\right)^2\bigg] \\
&\le \frac{2R^2 e^{4R} + 2}{\gam^2} \EE_{x\sim \mathcal{D},a\sim\pi^*(x)} \left[\Delta\bar{R}_{\hpi}(x,a)^2\right] \tag{Jensen's inequality; Assumption~\ref{assp:bdd}}\\
&\le \frac{16C_\Pi(R^2 e^{4R}+1)e^{4\gam R}}{\gam^2} \cd \frac{\ln(|\Pi|/\delta)}{n}~. \tag{by \eqref{eq:delta}}
\end{align*}
Putting everything together, we conclude:
\begin{align*}
&\EE_{x\sim\cD} \left[\KL(\pi^*(x)\Vert \hpi(x))\right]\\
&\le (\fr12\vee\psi(e^{4R})) \left(\frac{16C_\Pi e^{4\gam R}}{\gam^2} \cd \frac{\ln(|\Pi|/\delta)}{n} + \frac{32C_\Pi(R^2 e^{4R}+1)e^{4\gam R}}{\gam^2} \cd \frac{\ln(|\Pi|/\delta)}{n}\right) \\
&= (\fr12\vee\psi(e^{4R})) \frac{16(2R^2 e^{4R}+3)C_\Pi e^{4\gam R}}{\gam^2} \cd \frac{\ln(|\Pi|/\delta)}{n}.
\end{align*}
We remark that by Lemma \ref{lem:kl_bounded_by_log_sq}, the $\fr12\vee\psi(e^{4R})$ term is further bounded above by $\frac{e^{4R}}{16R^2}$.

\vspace{.5em}
\begin{remark}\label{rem:novelty}
    One of our key novelties is \eqref{eq:novelty}.
    In \citet{agarwal25design}, they use Cauchy-Schwarz to derive the bound
    \begin{align*}
        \EE_{x\sim \cD, a\sim \pi^*(x)} \sbr[2]{\ln \fr{\pi^*(a\mid x)}{\hpi(a \mid x)} }
        \le \sqrt{\EE_{x\sim \cD, a\sim \pi^*(x)} \sbr[2]{\del[2]{\ln \fr{\pi^*(a\mid x)}{\hpi(a \mid x)}}^2}} ~,
    \end{align*}
    which introduces an extra square root compared to our derivation.
    Following their approach naively would lead to a $1/\sqrt{n}$ rate instead of $1/n$.
\end{remark}
\qed

\subsection{Proofs for Section~\ref{sec:pref}}\label{app:proof_distill}

\textit{Proof of Theorem \ref{thm:pref_distill}.} Up to constants, our distillation objective is equivalent to minimizing
\begin{align*}
\frac{1}{n} \sum_{(x,a^+,a^-)\in D_n} \KL\left(\text{Bern}(\PP_{\wtpi}(a^+\succ a^-\mid x)) \Vert \text{Bern}(\PP_\pi(a^+\succ a^-\mid x))\right),
\end{align*}
which can achieve zero loss since $\wtpi\in\cP_\gam(\cR) \subseteq\Pi$ is a valid solution. Thus, the solution $\hpi$ must satisfy
\begin{align*}
    \PP_{\wtpi}(a\succ b\mid x) = \PP_{\hpi}(a\succ b\mid x), \quad\forall (x,a,b)\in D_n
\end{align*}
(recall that we use $(a,b)$ to denote the independent unlabeled responses). Defining the set
\begin{align*}
\cK := \left\{(\pi_1,\pi_2)\in\cP_\gam(\cR) \times\Pi: \EE_{x\sim\cD,a,b\sim\mu(x)}\big[|\PP_{\pi_1}(a\succ b\mid x) - \PP_{\pi_2}(a\succ b\mid x)|\big] > \epsilon\right\},
\end{align*}
it follows that
\begin{align*}
\PP\left((\wtpi,\hpi)\in \cK\right) &= \sum_{(\pi_1,\pi_2)\in \cK} \PP(\wtpi=\pi_1,\hpi =\pi_2) \\
&\le \sum_{(\pi_1,\pi_2)\in \cK} \PP\left(\PP_{\pi_1}(a\succ b\mid x) = \PP_{\pi_2}(a\succ b\mid x), \;\forall (x,a,b)\in D_n\right) \\
&= \sum_{(\pi_1,\pi_2)\in \cK} \PP\left(\PP_{\pi_1}(a\succ b\mid x) = \PP_{\pi_2}(a\succ b\mid x)\right)^n\\
&\le \sum_{(\pi_1,\pi_2)\in \cK} \left(1- \EE\left[|\PP_{\pi_1}(a\succ b\mid x) - \PP_{\pi_2}(a\succ b\mid x)|\right]\right)^n \\
&\le \sum_{(\pi_1,\pi_2)\in \cK} (1-\eps)^n \\
&\le |\cK|^2\exp(-\eps n).
\end{align*}
Therefore $\PP\left((\wtpi,\hpi)\in \cK\right) \le |\Pi|^2 \exp(-\eps n)$, i.e.,
\begin{align*}
\EE_{x\sim\cD,a,b\sim\mu(x)}\big[|\PP_{\wtpi}(a\succ b\mid x) - \PP_{\hpi}(a\succ b\mid x)|\big] \le \frac{2\ln(|\Pi|/\delta)}{n}
\end{align*}
with probability at least $1-\delta$, and so
\begin{align}\label{eq:new_distill}
\EE_{x\sim\cD,a,b\sim\mu(x)}\left[(\PP_{\wtpi}(a\succ b\mid x) - \PP_{\hpi}(a\succ b\mid x))^2\right] \le \frac{2\ln(|\Pi|/\delta)}{n}
\end{align}
as well.
On the other hand, applying Proposition~\ref{thm:oracle} to $\cP_\gam(\cR)$, we have
\begin{align}\label{eq:new_reward}
\EE_{x\sim \mathcal{D},a,b\sim\mu(x)}\left[(\PP_{\wtpi}(a \succ b \mid x) - \PP_*(a \succ b \mid x))^2\right] \le \frac{4\ln(|\cR|/\delta)}{n}
\end{align}
with probability at least $1-\delta$.
Hence by a union bound, it holds that, with probability at least $1-\delta$,
\begin{align*}
\EE_{x\sim \mathcal{D},a,b\sim\mu(x)}\left[(\PP_{\hpi}(a \succ b \mid x) - \PP_*(a \succ b \mid x))^2\right] \le \frac{4\ln(2|\Pi|/\delta) + 8\ln(2|\cR|/\delta)}{n}~.
\end{align*}

Furthermore, by Lemma \ref{lem:sig_diff} it holds that
\begin{align*}
\abs{\PP_{\hpi}(a \succ b \mid x) - \PP_*(a \succ b \mid x)} &= \abs{\sig\bigg(\gam\ln\frac{\hpi(a\mid x)}{\hpi(b\mid x)}\bigg) - \sig\bigg(\gam\ln\frac{\pi^*(a\mid x)}{\pi^*(b\mid x)}\bigg)} \\
&\ge \frac{\gam e^{-2\gam R}}{2} \abs{\ln\frac{\hpi(a\mid x)}{\hpi(b\mid x)} - \ln\frac{\pi^*(a\mid x)}{\pi^*(b\mid x)}} \\
&=\frac{e^{-2\gam R}}{2} \abs{\Delta\bar{R}_{\hpi}(x,a) - \Delta\bar{R}_{\hpi}(x,b)},
\end{align*}
which implies that
\begin{align}
\EE_{x\sim \cD, a\sim \mu(x)} \left[(\Delta \bar R_{\hpi} (x,a))^2\right] &= \fr12 \EE_{x\sim \cD, a,b\sim \mu(x)} \left[(\Delta \bar R_{\hpi} (x,a) - \Delta \bar R_{\hpi} (x,b))^2 \right]\nonumber \\
&\le 2e^{4\gam R} \cdot \frac{4\ln(2|\Pi|/\delta) + 8\ln(2|\cR|/\delta)}{n}.\label{eq:delta_bound_pref}
\end{align}
Finally as in the proof of Theorem \ref{thm:mle}, we combine the bounds
\begin{align*}
&\EE_{x\sim\cD, a\sim\pi^*(x)} \bigg[\bigg(\ln \frac{\hpi(a\mid x)}{\pi^*(a\mid x)}\bigg)^2 \bigg]\\
&\le \frac{2}{\gam^2} \EE_{x\sim \mathcal{D},a\sim\pi^*(x)} \left[\Delta\bar{R}_{\hpi}(x,a)^2\right] + 2\EE_{x\sim\cD} \bigg[\left(\ln\frac{Z_*(x)}{Z_{\hpi}(x)}\right)^2 \bigg]
\end{align*}
and
\begin{align*}
\EE_{x\sim\cD} \bigg[\left(\ln\frac{Z_*(x)}{Z_{\hpi}(x)}\right)^2 \bigg]\le \frac{2R^2 e^{4R} + 2}{\gam^2} \EE_{x\sim \mathcal{D},a\sim\pi^*(x)} \left[\Delta\bar{R}_{\hpi}(x,a)^2\right]
\end{align*}
along with \eqref{eq:delta_bound_pref} to conclude that
\begin{align*}
\EE_{x\sim\cD}\left[\KL(\pi^*(x)\Vert \hpi(x))\right] \le (\fr12\vee\psi(e^{4R})) \frac{16(2R^2 e^{4R}+3) C_\Pi e^{4\gam R}}{\gam^2} \cd \frac{\ln(2|\Pi|/\delta) + 2\ln(2|\cR|/\delta)}{n}.
\end{align*}
\qed

\begin{theorem}
\label{thm:banana}
Given a base policy $\pi_0$, denote by $C_0$ the smallest constant such that for every $\pi,\pi'\in\Pi$,\footnote{This is slightly stronger than Definition \ref{def:cov}, which can be retrieved by setting $\pi'=\pi^*$. We can weaken the definition to include only $\pi'\in\cR$.}
\begin{align*}
\EE_{x\sim \cD, a\sim \pi^*(x)}\big[(\Delta \bar R_\pi(x,a) - \Delta \bar R_{\pi'}(x,a))^2\big] \le C_0 \EE_{x\sim \cD, a\sim \pi_0(x)}\big[(\Delta \bar R_\pi(x,a) - \Delta \bar R_{\pi'}(x,a))^2\big].
\end{align*}

The preference distillation estimate $\hpi = \argmin\nolimits_{\pi\in\Pi} \cL_{\distill}(\pi)$ with responses in $D_n$ generated from $\pi_0$ instead of $\mu$ satisfies
\begin{align}\label{eq:new_distill_guarantee}
\EE_{x\sim\cD}\left[\KL(\pi^*(x)\Vert \hpi(x))\right] \lesssim \frac{1}{\gam^2} \cd \frac{C_0\ln(|\Pi|/\delta) + C_\cR\ln(|\cR|/\delta)}{n}
\end{align}
with probability at least $1-\delta$.
\end{theorem}

\begin{proof}
In the proof of Theorem \ref{thm:pref_distill}, \eqref{eq:new_distill} now holds with $\mu$ replaced by $\pi_0$, while \eqref{eq:new_reward} remains unchanged. It follows from Lemma \ref{lem:sig_diff} that
\begin{align*}
&\frac{2\ln(|\Pi|/\delta)}{n}\\
&\ge \EE_{x\sim\cD,a,b\sim\pi_0(x)}\left[(\PP_{\wtpi}(a\succ b\mid x) - \PP_{\hpi}(a\succ b\mid x))^2\right] \\
&\ge \frac{\gam^2 e^{-4\gam R}}{4}\EE_{x\sim\cD,a,b\sim\pi_0(x)}\left[\left(\ln\frac{\hpi(a\mid x)}{\hpi(b\mid x)} - \ln\frac{\wtpi(a\mid x)}{\wtpi(b\mid x)}\right)^2\right] \\
&= \frac{e^{-4\gam R}}{4} \EE_{x\sim\cD,a,b\sim\pi_0(x)} \left[\left(\Delta\bar{R}_{\hpi}(x,a) - \Delta\bar{R}_{\hpi}(x,b) - \Delta\bar{R}_{\wtpi}(x,a) + \Delta\bar{R}_{\wtpi}(x,b)\right)^2\right] \\
&= \frac{e^{-4\gam R}}{2} \EE_{x\sim\cD,a\sim\pi_0(x)} \left[\left(\Delta\bar{R}_{\hpi}(x,a) - \Delta\bar{R}_{\wtpi}(x,a)\right)^2\right] \\
&\ge \frac{e^{-4\gam R}}{2C_0} \EE_{x\sim\cD,a\sim\pi^*(x)} \left[\left(\Delta\bar{R}_{\hpi}(x,a) - \Delta\bar{R}_{\wtpi}(x,a)\right)^2\right],
\end{align*}
where the last line uses the modified definition of the coverage coefficient $C_0$. Moreover, \eqref{eq:new_reward} implies
\begin{align*}
\EE_{x\sim \cD, a\sim \pi^*(x)} \left[(\Delta \bar R_{\wtpi} (x,a))^2\right] &\le C_\cR
\EE_{x\sim \cD, a\sim \mu(x)} \left[(\Delta \bar R_{\wtpi} (x,a))^2\right]\\
&= \fr12 C_\cR \EE_{x\sim \cD, a,b\sim \mu(x)} \left[(\Delta \bar R_{\wtpi} (x,a) - \Delta \bar R_{\wtpi} (x,b))^2 \right] \\
&\le 2C_\cR e^{4\gam R} \cdot \frac{4\ln(|\cR|/\delta)}{n}.
\end{align*}
By a union bound, it follows with probability at least $1-\delta$ that
\begin{align*}
&\EE_{x\sim \cD, a\sim \pi^*(x)} \left[(\Delta \bar R_{\hpi} (x,a))^2\right]\\
&\le 2\EE_{x\sim\cD,a\sim\pi^*(x)} \left[\left(\Delta\bar{R}_{\hpi}(x,a) - \Delta\bar{R}_{\wtpi}(x,a)\right)^2\right] + 2\EE_{x\sim \cD, a\sim \pi^*(x)} \left[(\Delta \bar R_{\wtpi} (x,a))^2\right] \\
&\le 8e^{4\gam R}\cdot \frac{C_0 \ln(2|\Pi|/\delta) + 2C_\cR\ln(2|\cR|/\delta)}{n}.
\end{align*}
The remainder of the proof proceeds similarly.
\end{proof}

\subsection{Proofs for Section \ref{sec:revkl}}\label{app:proof_reverse}

\textit{Proof of Theorem \ref{thm:reverse_kl}.}
The first step of the argument is similar to the proof of Theorem \ref{thm:pref_distill}, but with a stronger sup norm bound which is guaranteed due to realizability. Indeed, up to constants, the reverse KL objective is equivalent to
\begin{align*}
\hpi = \argmin_{\pi\in\Pi} \frac{1}{n} \sum\limits_{(x,\cd,\cd) \in D_n} \KL(\pi(x)\Vert\wtpi(x)),
\end{align*}
which achieves zero loss due to Assumption \ref{assp:reward_model}. Defining the set
\begin{align*}
K:= \Bigg\{(\pi_1,\pi_2)\in\Pi\times\cP_\gam(\cR) : \EE_{x\sim\cD} \bigg[\sup_{a\in\cA}\left(\ln\frac{\pi_1(a\mid x)}{\pi_2(a\mid x)}\right)^2\bigg] >\eps\Bigg\},
\end{align*}
it follows that
\begin{align*}
\PP\left(\EE_{x\sim\cD} \bigg[\sup_{a\in\cA}\left(\ln\frac{\hpi(a\mid x)}{\wtpi(a\mid x)}\right)^2\bigg] >\eps\right)
&= \sum_{(\pi_1,\pi_2)\in K} \PP(\hpi=\pi_1,\wtpi=\pi_2) \tag{law of total probability}\\
&\le \sum_{(\pi_1,\pi_2)\in K} \PP(\pi_1(x) = \pi_2(x), \;\forall x\in D_n) \\
&= \sum_{(\pi_1,\pi_2)\in K} \PP_{x\sim\cD}(\pi_1(x) = \pi_2(x))^n.
\end{align*}
Note that for any $(\pi_1,\pi_2)\in K$, it holds that $\pi_1(a\mid x)/\pi_2(a\mid x)\le e^{4R}$ as in the proof of Theorem \ref{thm:mle}, so that
\begin{align*}
\sup_{a\in\cA} \left(\ln\frac{\pi_1(a\mid x)}{\pi_2(a\mid x)}\right)^2 \le 16R^2\cd 1_{\{\pi_1(x)\neq\pi_2(x)\}}, \quad\forall x\in\cX.
\end{align*}
This implies
\begin{align*}
\PP_{x\sim\cD}(\pi_1(x) = \pi_2(x)) &= 1-\EE_{x\sim\cD}[1_{\{\pi_1(x)\neq\pi_2(x)\}}] \le 1 - \frac{\eps}{16R^2},
\end{align*}
and hence
\begin{align}\label{eq:summer}
\PP\left(\EE_{x\sim\cD} \bigg[\sup_{a\in\cA}\left(\ln\frac{\hpi(a\mid x)}{\wtpi(a\mid x)}\right)^2\bigg] >\eps\right) \le |K| \left(1-\frac{\eps}{16R^2}\right)^n \le |\Pi|^2 \exp\left(-\frac{\eps n}{16R^2}\right).
\end{align}
We now bound the forward KL divergence. Again applying Lemma \ref{lem:kl_bounded_by_log_sq},
\begin{align*}
&\EE_{x\sim\cD}\left[\KL(\pi^*(x)\Vert \hpi(x))\right]\\
&\le (\fr12\vee\psi(e^{4R})) \EE_{x\sim\cD, a\sim\pi^*(x)} \bigg[\bigg(\ln \frac{\hpi(a\mid x)}{\pi^*(a\mid x)}\bigg)^2 \bigg]\\
&\le (1\vee 2\psi(e^{4R})) \left(\EE_{x\sim\cD} \bigg[\sup_{a\in\cA}\bigg(\ln \frac{\hpi(a\mid x)}{\wtpi(a\mid x)}\bigg)^2 \bigg] + \EE_{x\sim\cD, a\sim\pi^*(x)} \bigg[\bigg(\ln \frac{\wtpi(a\mid x)}{\pi^*(a\mid x)}\bigg)^2 \bigg] \right).
\end{align*}
The first term can be bounded via \eqref{eq:summer}. For the second term, repeating the derivation of Theorem \ref{thm:mle} over the policy class $\cP_\gam(\cR)$ instead of $\Pi$, we obtain
\begin{align*}
\EE_{x\sim\cD, a\sim\pi^*(x)} \bigg[\bigg(\ln \frac{\wtpi(a\mid x)}{\pi^*(a\mid x)}\bigg)^2 \bigg] \le \frac{16C_{\cR}(R^2 e^{4R}+1)e^{4\gam R}}{\gam^2} \cd \frac{\ln(|\cR|/\delta)}{n}.
\end{align*}
Putting everything together, we conclude:
\begin{align*}
&\EE_{x\sim\cD}\left[\KL(\pi^*(x)\Vert \hpi(x))\right]\\
&\le (1\vee 2\psi(e^{4R})) \left(16R^2 \cd \frac{\ln(2|\Pi|^2/\delta)}{n} + \frac{16C_{\cR}(R^2 e^{4R}+1)e^{4\gam R}}{\gam^2} \cd \frac{\ln(2|\cR|/\delta)}{n}\right)
\end{align*}
with probability at least $1-\delta$.
\qed

\begin{proposition}\label{prop:klsame}
It holds that
\begin{align*}
  \EE_x[\KL(\hpi(x) \Vert \pi^*(x))] \le  \fr{(4R \vee 1)e^{8R+1}}{R^2} \EE_x[\KL(\pi^*(x) \Vert \hpi(x))].
\end{align*}
\end{proposition}

\begin{proof}[Proof of Proposition \ref{prop:klsame}]
Denote the ratio $r = \fr{\pi^*(a\mid x)}{\hpi(a\mid x)}$ for brevity. Using the same argument as~\eqref{eq:prob_ratio}, we have $r \in [e^{-4R}, e^{4R}]$.
Then, we have
\begin{align*}
  r - 1 - \ln r
  \sr{(a)}{\le} \fr{e^{4R}}{R^2} (\ln r)^2
  =   \fr{e^{4R}}{R^2} \left(\ln\fr{1}{r}\right)^2
  \sr{(b)}{\le} \fr{e^{4R}}{R^2} e (4R \vee 1) \left(\fr1r - 1 - \ln\fr{1}{r}\right),
\end{align*}
where $(a)$ is by Lemma~\ref{lem:kl_bounded_by_log_sq} and $(b)$ is by Lemma~\ref{lem:opposite} with $r$ replaced by $1/r$.
Thus,
\begin{align*}
  \KL(\hpi(x) \Vert \pi^*(x))
  &= \EE_{a\sim \hpi(x)}[r - 1 - \ln r]
\\&\le  \fr{e^{4R}}{R^2} e (4R \vee 1) \EE_{a\sim \hpi(x)}\left[\fr1r - 1 - \ln\fr{1}{r}\right]
\\&\le  \fr{e^{4R}}{R^2} e (4R \vee 1) e^{4R}\EE_{a\sim \pi^*(x)} \left[\fr1r - 1 - \ln\fr{1}{r}\right]
\\&=  \fr{e^{4R}}{R^2} e (4R \vee 1) e^{4R} \KL(\pi^*(x) \Vert \hpi(x)).
\end{align*}

\end{proof}

\section{Experimental Details}\label{app:exp_details}

In Section~\ref{app:toy}, we describe the detailed settings for our toy experiments. In Section~\ref{app:tldr}, we provide implementation details on model card, hyperparameters, and compute resources on training and evaluating on the TL;DR dataset.
In Section~\ref{app:general}, we provide details on our general chat experiments from Section~\ref{sec:exp_main} and also show additional results on MT-Bench and AlpacaEval 2.0.

\subsection{Toy Experiments}\label{app:toy}

We adopt a tabular setting with vocabulary and context sizes both equal to $10$. Under the preference model \eqref{eq:preference-model-intro}, we fix $\gamma = 0.5$. To keep the oracle policy $\pi^*$ and reference policy $\piref$ close, for each context $x\in \mathcal{X}$, we draw logits $\alpha^*, \alpha_0 \in \mathbb{R}^{10}$ with entries i.i.d. $\mathcal{N}(0, 0.1^2)$ and set $\pi^*(\cdot \mid x) = \mathrm{softmax}(\alpha^*)$ and $\piref(\cdot \mid x) = \mathrm{softmax}(\alpha^0)$. We then train polices by optimizing \eqref{eq:rlhf} and \eqref{eq:rkl} and evaluate the forward KL $\mathbb{E}_x\big[\mathrm{KL}(\pi^*(\cdot \mid x) \Vert \pi(\cdot \mid x)\big]$. We vary the sample size $n \in \{2^5, \cdots, 2^{16}\}$, tuning the KL coefficient $\beta$ for both RLHF and RKL at each $n$, while keeping $\gamma = 0.5$ fixed for RKL.

\subsection{TL;DR Summarization}
\label{app:tldr}

\paragraph{Dataset.} We use the TL;DR dataset that is widely used in related literature~\citep{gao24rebel,song24theimportance,huang2024the}, publicly available\footnote{\url{https://github.com/openai/summarize-from-feedback}}. We summarize the dataset statistics in Table \ref{tab:tldr_stats}.
Note that DPO and PMLE are trained on the preference dataset which has preference labels, and other algorithms evaluate the policy based on human references since they utilize the online responses.

\begin{table}[htb!]
    \caption{TL;DR dataset statistics.}\label{tab:tldr_stats}
    \centering
    \begin{tabular}{@{}c|ccc@{}}
        \toprule
        Dataset & Train & Valid & Test \\ \midrule \midrule
        Human Reference & 117K & 64.5K & 6.55K  \\
        Preference & 92.9K & 83.8K & N/A \\
        \bottomrule
    \end{tabular}
\end{table}

\paragraph{Models.}
We use Pythia-2.8B\footnote{\url{https://huggingface.co/EleutherAI/pythia-2.8b-deduped}} and Pythia-6.9B\footnote{\url{https://huggingface.co/EleutherAI/pythia-6.9b-deduped}} \citep{biderman2023pythia} as our pretrained models, using maximum context length $512$ and maximum generation length up to $53$ tokens. In order for training efficiency, we use LoRA (Low-Rank Adapter, \citet{hu2022lora}) for alignment after full-parameter tuning the SFT model.

\paragraph{Implementations.}
We implement our three approaches (PMLE, reverse KL, preference distillation) on the top of a publicly available codebase\footnote{\url{https://github.com/vwxyzjn/summarize_from_feedback_details}}; preference distillation in particular is based on another publicly available code baseline\footnote{\url{https://github.com/ZhaolinGao/REBEL}}.
For PMLE (Section \ref{sec:mle}), we implement the KL regularizer in \eqref{eq:mle-objective} using the online responses described in~\citet{song24theimportance}. The DPO baseline takes about $3$ hours and PMLE requires about $6$ hours with 4 A100 40GB GPUs. Also, reverse KL and preference distillation, as well as their corresponding baselines RLHF and REBEL, takes about $2.5$ days with 4 A100 40GB GPUs. Lastly, the win-rate is judged by GPT-4 using the \texttt{gpt-4} checkpoint (as of May 23rd, 2025).

\begin{algorithm}[htb!]
    \caption{Preference Distillation (Sec. \ref{sec:pref})}\label{alg:pref_distill}
    \begin{algorithmic}[1]
        \STATE {\bfseries Input: } Reward model $\hat{R}$, policy class $\Pi$, sampling distribution $\mu$, learning rate $\eta$, prompt dataset $\{x_i\}_{i=1}^{n}$
        \FOR{$t = 0, 1, \ldots, T-1$}
            \STATE Sample two responses from $a_1, a_2 \sim \mu(\cdot \mid x)$ for a given prompt $x \sim \mathcal{D}$ for all $x \in \{x_i\}_{i=1}^{n}$
            \STATE Compute the probabilities with preference simulator by
            \begin{align*}
                \mathbb{P}_{\tilde\pi}(a_1 \succ a_2 \mid x) & := \sigma\big(\hat{R}(x, a_1) - \hat{R}(x, a_2)\big) \\
                \mathbb{P}_{\tilde\pi}(a_2 \succ a_1 \mid x) & := \sigma\big(\hat{R}(x, a_2) - \hat{R}(x, a_1)\big)
            \end{align*}
            \STATE Compute the preference distillation loss $\cL_{\distill,\beta}(\pi)$ using \eqref{eq:distill} and \eqref{eq:pref_distill}
            \STATE $\pi_{t+1} \gets \pi_t - \eta \nabla \cL_{\distill,\beta}(\pi_t)$
	\ENDFOR
    \end{algorithmic}
\end{algorithm}

\paragraph{Pseudocode.} Since the implementations of PMLE and reverse KL are straightforward from the corresponding DPO and RLHF baseline, we present the pseudocode for preference distillation for better understanding. As noted in~\citet{gao24rebel}, the base distribution $\mu$ can also be $\pi_t$ in our pseudocode (Algorithm \ref{alg:pref_distill}). Following the baseline code implementations of REBEL, we also sample online responses from the distribution $\pi_t$.

\paragraph{Hyperparameters.}
We adopt almost the same hyperparameters used in several studies~\citep{huang2024the,gao24rebel,song24theimportance}. For completeness, we summarize the hyperparameters used in our experiments in Table \ref{tab:tldr_hparams}.
Note that \citet{gao24rebel} trains only a single epoch for RLHF and REBEL, but we cannot reproduce their results with just one epoch.
Rather, following the implementation details~\citep{huang2024the}, we consider the total episodes $10^6$ which corresponds to roughly about $8.5$ epochs. In this setting, we could reproduce the baseline results or obtain better results. Hence, reverse KL and preference distillation  are also evaluated under this setup.

\begin{table}[t]
    \centering
    \caption{Hyperparameter configurations for TL;DR summarization tasks.}
    \label{tab:tldr_hparams}
    \resizebox{\linewidth}{!}{
    \begin{tabular}{p{0.25\linewidth}p{0.375\linewidth}p{0.375\linewidth}}
        \midrule[0.3ex]
        \textbf{Setting} & \textbf{Parameters} & \\
        \midrule[0.15ex]
        SFT \& RM &
        batch size: 64 \newline
        learning rate: 3e-6 &
        schedule: cosine decay \newline
        train epochs: 1 \\
        \midrule[0.15ex]\midrule[0.15ex]
        DPO &
        batch size: 64 \newline
        learning rate: 3e-6 \newline
        schedule: linear decay &
        train epochs: 1 \newline
        $\beta$: 0.05 \\
        \midrule[0.15ex]
        PMLE &
        batch size: 512 \newline
        learning rate: 1e-6 \newline
        schedule: linear decay  &
        train epochs: 1 \newline
        $\beta$: 1e-5 \newline
        $\gamma$: 1e-2 \\
        \midrule[0.15ex]
        REBEL &
        batch size: 512 \newline
        learning rate: 3e-6 \newline
        schedule: linear decay \newline
        total episodes: 1e6 &
        num epochs: 4 \newline
        $\eta$: 1.0 \newline
        kl coefficient: 0.05 \\
        \midrule[0.15ex]
        Preference Distillation &
        batch size: 512 \newline
        learning rate: 3e-6 \newline
        schedule: linear decay \newline
        total episodes: 1e6 &
        num epochs: 4 \newline
        $\gamma$: 0.1 \newline
        kl coefficient: 0.05 \\
        \midrule[0.15ex]
        RLHF (via PPO) &
        batch size: 512 \newline
        learning rate: 3e-6 \newline
        schedule: linear decay \newline
        total episodes: 1e6 \newline
        num epochs: 4 &
        discount factor: 1 \newline
        gae $\lambda$: 0.95 \newline
        clip ratio: 0.2 \newline
        value function coeff: 0.1 \newline
        kl coefficient: 0.05 \\
        \midrule[0.15ex]
        Reverse KL (Sec. \ref{sec:revkl}) &
        batch size: 512 \newline
        learning rate: 3e-6 \newline
        schedule: linear decay  &
        total episodes: 1e6 \newline
        kl coefficient: 0.05 \newline
        entropy coefficient: 0.01 \\
        \midrule[0.15ex]\midrule[0.15ex]
        LoRA Adapter \newline Config &
        r: 1024 \newline
        $\alpha$: 2048 &
        dropout: 0.0 \newline
        bias: False \\
        \midrule[0.15ex]
        Generation \newline Config &
        sampling: true \newline
        top k: 0.0 \newline
        top p: 1.0 &
        min length: 53 \newline
        max new tokens: 53 \newline
        temperature: 0.1 (for DPO and PMLE) or 0.7 (others) \\
        \midrule[0.3ex]
    \end{tabular}}
\end{table}

\subsection{General Chat}
\label{app:general}

\paragraph{Dataset and Models.} In this experiment, we use the UltraFeedBack dataset~\citep{cui2023ultrafeedback}, which is used in various baselines. We use LLaMA-3-8B-Instruct\footnote{\url{https://huggingface.co/meta-llama/Meta-Llama-3-8B-Instruct}} as our base model and ArmoRM-Llama3-8B-v0.1\footnote{\url{https://huggingface.co/RLHFlow/ArmoRM-Llama3-8B-v0.1}} as the reward model. One can use other public base models and reward models as well.

\paragraph{Implementation.} As in the TL;DR experiments, our implementation for preference distillation is based on~\citet{gao24rebel}, which is publicly available. The total training time for LLaMA-3-8B-Instruct takes around 7 days on 4 A100 GPUs.

\paragraph{Hyperparameters.}

Similar to TL;DR experiments, we adopt the hyperparemeter configurations of~\citet{gao24rebel}; for completeness, the full specification is presented in Table \ref{tab:chat_hparams}. Note that, due to the scale of experiments, we choose the best hyperparameter $\gamma$ used in TL;DR experiments. Consequently, the reported performance of preference distillation might be conservative, as more fine-grained hyperparameter tuning could yield further performance gains.

\begin{table}[htb!]
    \centering
    \caption{Hyperparameter configurations for general chat experiments.}
    \label{tab:chat_hparams}
    \resizebox{\linewidth}{!}{
    \begin{tabular}{p{0.25\linewidth}p{0.375\linewidth}p{0.375\linewidth}}
        \midrule[0.3ex]
        \textbf{Setting} & \textbf{Parameters} & \\
        \midrule[0.15ex]
        DPO &
        batch size: 128 \newline
        learning rate: 5e-7 \newline
        schedule: cosine decay \newline
        train epochs: 1 \newline
        $\beta$: 0.05 \\
        \midrule[0.15ex]
        PMLE &
        batch size: 128 \newline
        learning rate: 3e-7 \newline
        schedule: cosine decay \newline
        train epochs: 1 \newline
        $\gamma$: 5e-3 \newline
        KL coefficient: 0.0 \\
        \midrule[0.15ex]
        REBEL &
        batch size: 128 \newline
        learning rate: 3e-7 \newline
        schedule: cosine decay \newline
        train epochs: 1 \newline
        num epochs: 1 \newline
        $\eta$: $10^6$ \\
        \midrule[0.15ex]
        Preference distillation &
        batch size: 128 \newline
        learning rate: 3e-7 \newline
        schedule: cosine decay  \newline
        train epochs: 1 \newline
        $\beta$: 0.05 \\
        \midrule[0.15ex]\midrule[0.15ex]
        Generation \newline Config & sampling: true \newline top k: 0.0 \newline top p: 0.9 & min length: 1024 \newline max new tokens: 1024 \newline temperature: 0.8 \\
        \midrule[0.3ex]
    \end{tabular}}
\end{table}



\end{document}

%% file: defs-kjun-v6-overleaf.tex
\def\safedef#1{%
   \ifx#1\undefined
      \expandafter\def\expandafter#1%
   \else
      \errmessage{The \string#1 is defined already}%
      \expandafter\def\expandafter\tmp
   \fi
}

\usepackage{amsthm,amsmath,bbm,amsfonts,amssymb}
\usepackage{dsfont} 
\usepackage{mathtools}
\usepackage{commath}
\mathtoolsset{showonlyrefs=false}
\allowdisplaybreaks 

\usepackage{xspace}
\usepackage[T1]{fontenc}
\usepackage[utf8]{inputenc}
\usepackage[dvipsnames]{xcolor} 

\usepackage{hyperref,url}

\usepackage{wrapfig}
\usepackage{pbox} 

\usepackage{enumitem}
\usepackage[normalem]{ulem}

{\color{kjsaved}}%

{\color{kjsavedenvkj}}%

{\color{kjsavedrevised}}%
\usepackage{framed}
\usepackage[most]{tcolorbox}
\definecolor{kjgray}{rgb}{.7,.7,.7}

\newtheoremstyle{kjstyle}
{1ex} 
{\topsep} 
{\itshape} 
{} 
{\bfseries} 
{.} 
{.5em} 
{} 

\newtheoremstyle{kjstyle2}
{.0em} 
{.0em} 
{\itshape} 
{} 
{\bfseries} 
{.} 
{.5em} 
{} 

\newtheoremstyle{kjstylenoitalic}
{1ex} 
{\topsep} 
{} 
{} 
{\bfseries} 
{.} 
{.5em} 
{} 

\tcbset{kjboxstyle/.style={title={},breakable,colback=white,enhanced jigsaw,boxrule=1.3pt,sharp corners,colframe=kjgray,boxsep=0pt,coltitle={black},attach title to upper={},left=.8ex,bottom=.4em}}    

\tcbset{kjboxstylec/.style={title={},breakable,colback=white,enhanced jigsaw,boxrule=1.3pt,sharp corners,colframe=kjgray,boxsep=0pt,coltitle={black},attach title to upper={},before skip=1.5ex,left=.8ex,bottom=.4em,top=0.4em,enlarge top by=-0.3em,enlarge bottom by=-0.3em}}    

\usepackage{mdframed}
\usepackage{lipsum}
\definecolor{kjgray}{rgb}{.7,.7,.7}
\makeatletter



\makeatletter
\renewcommand{\paragraph}{%
  \@startsection{paragraph}{4}%
  {\z@}{0.50ex \@plus 1ex \@minus .2ex}{-1em}%
  {\normalfont\normalsize\bfseries}%
}
\makeatother

\usepackage{tabularx} 
\newcolumntype{P}[1]{>{\centering\arraybackslash}p{#1}}
\newcolumntype{M}[1]{>{\centering\arraybackslash}m{#1}}

\def\ddefloop#1{\ifx\ddefloop#1\else\ddef{#1}\expandafter\ddefloop\fi}

\def\ddef#1{\expandafter\def\csname #1#1\endcsname{\ensuremath{\mathbb{#1}}}}
\ddefloop ABCDFGHIJKLMNORSTUWXYZ\ddefloop 

\def\ddef#1{\expandafter\def\csname c#1\endcsname{\ensuremath{\mathcal{#1}}}}
\ddefloop ABCDEFGHIJKLMNOPQRSTUVWXYZ\ddefloop

\def\ddef#1{\expandafter\def\csname b#1\endcsname{\ensuremath{{\mathbf{#1}}}}}
\ddefloop ABCDEFGHIJKLMNOPQRSTUVWXYZ\ddefloop  
\def\ddef#1{\expandafter\def\csname b#1\endcsname{\ensuremath{{\boldsymbol{#1}}}}}
\ddefloop abcdeghijklmnopqrtsuvwxyz\ddefloop  

\def\ddef#1{\expandafter\def\csname h#1\endcsname{\ensuremath{\hat{#1}}}}
\ddefloop ABCDEFGHIJKLMNOPQRSTUVWXYZabcdefghijklmnopqrsuvwxyz\ddefloop 
\def\ddef#1{\expandafter\def\csname hc#1\endcsname{\ensuremath{\hat{\mathcal{#1}}}}}
\ddefloop ABCDEFGHIJKLMNOPQRSTUVWXYZ\ddefloop
\def\ddef#1{\expandafter\def\csname hb#1\endcsname{\ensuremath{\hat{\mathbf{#1}}}}}
\ddefloop ABCDEFGHIJKLMNOPQRSTUVWXYZ\ddefloop %
\def\ddef#1{\expandafter\def\csname hb#1\endcsname{\ensuremath{\hat{\boldsymbol{#1}}}}}
\ddefloop abcdefghijklmnopqrstuvwxyz\ddefloop %

\def\ddef#1{\expandafter\def\csname t#1\endcsname{\ensuremath{\tilde{#1}}}}
\ddefloop ABCDEFGHIJKLMNOPQRSTUVWXYZabcdefgijklmnpqtsuvwxyz\ddefloop 
\def\ddef#1{\expandafter\def\csname tc#1\endcsname{\ensuremath{\tilde{\mathcal{#1}}}}}
\ddefloop ABCDEFGHIJKLMNOPQRSTUVWXYZ\ddefloop
\def\ddef#1{\expandafter\def\csname tb#1\endcsname{\ensuremath{\tilde{\mathbf{#1}}}}}
\ddefloop ABCDEFGHIJKLMNOPQRSTUVWXYZ\ddefloop
\def\ddef#1{\expandafter\def\csname tb#1\endcsname{\ensuremath{\tilde{\boldsymbol{#1}}}}}
\ddefloop abcdefghijklmnopqrstuvwxyz\ddefloop %

\def\ddef#1{\expandafter\def\csname bar#1\endcsname{\ensuremath{\bar{#1}}}}
\ddefloop ABCDEFGHIJKLMNOPQRSTUVWXYZabcdefghijklmnopqrtsuvwxyz\ddefloop
\def\ddef#1{\expandafter\def\csname barc#1\endcsname{\ensuremath{\bar{\mathcal{#1}}}}}
\ddefloop ABCDEFGHIJKLMNOPQRSTUVWXYZ\ddefloop
\def\ddef#1{\expandafter\def\csname barb#1\endcsname{\ensuremath{\bar{\mathbf{#1}}}}}
\ddefloop ABCDEFGHIJKLMNOPQRSTUVWXYZ\ddefloop
\def\ddef#1{\expandafter\def\csname barb#1\endcsname{\ensuremath{\bar{\boldsymbol{#1}}}}}
\ddefloop abcdefghijklmnopqrstuvwxyz\ddefloop %

\def\ddef#1{\expandafter\def\csname war#1\endcsname{\ensuremath{\overline{#1}}}}
\ddefloop ABCDEFGHIJKLMNOPQRSTUVWXYZabcdefghijklmnopqrtsuvwxyz\ddefloop
\def\ddef#1{\expandafter\def\csname warc#1\endcsname{\ensuremath{\overline{\mathcal{#1}}}}}
\ddefloop ABCDEFGHIJKLMNOPQRSTUVWXYZ\ddefloop
\def\ddef#1{\expandafter\def\csname warb#1\endcsname{\ensuremath{\overline{\mathbf{#1}}}}}
\ddefloop ABCDEFGHIJKLMNOPQRSTUVWXYZ\ddefloop
\def\ddef#1{\expandafter\def\csname warb#1\endcsname{\ensuremath{\overline{\boldsymbol{#1}}}}}
\ddefloop abcdefghijklmnopqrstuvwxyz\ddefloop %

\def\sig{\sigma}

\def\gam{\gamma}

\def\eps{\varepsilon}
\def\epsilon{\varepsilon}

\def\Dt{\Delta}

\usepackage{pgffor}
\def\greeksymbols{alpha,beta,gamma,gam,delta,dt,eps,epsilon,zeta,eta,theta,th,iota,kappa,kap,lambda,lam,mu,nu,xi,pi,rho,sigma,sig,tau,phi,chi,psi,omega,om,Gamma,Gam,Delta,Dt,Theta,Th,Lambda,Lam,Pi,Sigma,Sig,Phi,Psi,Omega,Om}
\def\greeksymbolsnoeta{alpha,beta,gamma,gam,delta,dt,eps,epsilon,zeta,theta,th,iota,kappa,kap,lambda,lam,mu,nu,xi,pi,rho,sigma,sig,tau,phi,chi,psi,omega,om,Gamma,Gam,Delta,Dt,Theta,Th,Lambda,Lam,Pi,Sigma,Sig,Phi,Psi,Omega,Om} 

\foreach \x in \greeksymbolsnoeta{\expandafter\xdef\csname b\x\endcsname{\noexpand\ensuremath{\noexpand\boldsymbol{\csname \x\endcsname}}}}

\foreach \x in \greeksymbols{\expandafter\xdef\csname h\x\endcsname{\noexpand\ensuremath{\noexpand\hat{\csname \x\endcsname}}}}
\foreach \x in \greeksymbolsnoeta{\expandafter\xdef\csname hb\x\endcsname{\noexpand\ensuremath{\noexpand\hat{\noexpand\boldsymbol{ \csname \x\endcsname}}}}}

\foreach \x in \greeksymbols{\expandafter\xdef\csname bar\x\endcsname{\noexpand\ensuremath{\noexpand\bar{\csname \x\endcsname}}}}
\foreach \x in \greeksymbolsnoeta{%
\expandafter\xdef\csname barb\x\endcsname{\noexpand\ensuremath{\noexpand\bar{\noexpand\boldsymbol{ \csname \x\endcsname}}}}
}

\foreach \x in \greeksymbols{\expandafter\xdef\csname t\x\endcsname{\noexpand\ensuremath{\noexpand\tilde{\csname \x\endcsname}}}}
\foreach \x in \greeksymbolsnoeta{\expandafter\xdef\csname tb\x\endcsname{\noexpand\ensuremath{\noexpand\tilde{\noexpand\boldsymbol{ \csname \x\endcsname}}}}}

\providecommand{\normz}[2][-1]{
\ensuremath{\mathinner{
\ifthenelse{\equal{#1}{-1}}{ 
\!\left\|#2\right\|}{}
\ifthenelse{\equal{#1}{0}}{ 
\|#2\|}{}
\ifthenelse{\equal{#1}{1}}{ 
\bigl\|#2\bigr\|}{}
\ifthenelse{\equal{#1}{2}}{ 
\Bigl\|#2\Bigr\|}{}
\ifthenelse{\equal{#1}{3}}{ 
\biggl\|#2\biggr\|}{}
\ifthenelse{\equal{#1}{4}}{ 
\Biggl\|#2\Biggr\|}{}
}} 
}  

\providecommand{\floor}[2][-1]{
\ensuremath{\mathinner{
\ifthenelse{\equal{#1}{-1}}{ 
\!\left\lfloor#2\right\rfloor}{}
\ifthenelse{\equal{#1}{0}}{ 
\lfloor#2\rfloor}{}
\ifthenelse{\equal{#1}{1}}{ 
\!\bigl\lfloor#2\bigr\rfloor}{}
\ifthenelse{\equal{#1}{2}}{ 
\!\Bigl\lfloor#2\Bigr\rfloor}{}
\ifthenelse{\equal{#1}{3}}{ 
\!\biggl\lfloor#2\biggr\rfloor}{}
\ifthenelse{\equal{#1}{4}}{ 
\!\Biggl\lfloor#2\Biggr\rfloor}{}
}} 
}

\providecommand{\ceil}[2][-1]{
\ensuremath{\mathinner{
\ifthenelse{\equal{#1}{-1}}{ 
\!\left\lceil#2\right\rceil}{}
\ifthenelse{\equal{#1}{0}}{ 
\lceil#2\rceil}{}
\ifthenelse{\equal{#1}{1}}{ 
\!\bigl\lceil#2\bigr\rceil}{}
\ifthenelse{\equal{#1}{2}}{ 
\!\Bigl\lceil#2\Bigr\rceil}{}
\ifthenelse{\equal{#1}{3}}{ 
\!\biggl\lceil#2\biggr\rceil}{}
\ifthenelse{\equal{#1}{4}}{ 
\!\Biggl\lceil#2\Biggr\rceil}{}
}} 
}

\newcommand{\fr}[2]{ { \frac{#1}{#2} }}

\def\cd{\cdot}

\def\rarrow{\ensuremath{\rightarrow}}

\definecolor{mygrn}{rgb}{0,.8,0}
\definecolor{myred}{rgb}{.8,0,0}

\DeclareMathOperator{\EE}{\mathbb{E}} 

\DeclareMathOperator{\PP}{\mathbb{P}}

\DeclareMathOperator*{\argmin}{arg~min}

\DeclareMathOperator{\KL}{{\mathrm{KL}}}

\DeclarePairedDelimiterX{\inp}[2]{\langle}{\rangle}{#1, #2}

\newcommand\declareop[3]{%
  \newcommand#1{%
    \mskip\muexpr\medmuskip*#2\relax
    {#3}%
    \mskip\muexpr\medmuskip*#2\relax
}}
\declareop\capprox{1}{{\sr{\const}{\approx}}} 
\declareop\logapprox{1}{{\sr{\mathrm{log}}{\approx}}} 

\def\Bernoulli{{\ensuremath{\mathrm{Bernoulli}}}}

\def\const{\mathsf{const}}

\def\tmin{{\min}}
\def\tmax{{\max}}

\usepackage{pifont}
\newcommand{\sr}{\stackrel}

\makeatletter
\newcommand{\vast}{\bBigg@{3}}
\newcommand{\Vast}{\bBigg@{4}}
\makeatother

\newenvironment{talign*}
 {\csname align*\endcsname}
 {\endalign}

\def\chrulefill{\leavevmode\leaders\hrule height 0.7ex depth \dimexpr0.4pt-0.7ex\hfill\kern0pt}

\renewcommand{\cite}[1]{\citep{#1}}

\newcommand{\vecp}{\mathrm{\mathbf{p}}}
\newcommand{\vecr}{\mathrm{\mathbf{r}}}
\def\pmle{\textup{\textsf{PMLE}}}

\def\rkl{\textup{\textsf{RKL}}}
\def\rlhf{\textup{\textsf{RLHF}}}
\def\distill{\textup{\textsf{Distill}}}

\def\wtpi{\widetilde{\pi}}

\def\rmin{r_\tmin}
\def\rmax{r_\tmax}

\theoremstyle{plain}
\newtheorem{theorem}{Theorem} 
\newtheorem{proposition}[theorem]{Proposition}
\newtheorem{lemma}[theorem]{Lemma}

\theoremstyle{definition}
\newtheorem{definition}[theorem]{Definition}
\newtheorem{assumption}[theorem]{Assumption}
\theoremstyle{remark}
\newtheorem{remark}[theorem]{Remark}

%% file: neurips_2026.bib
@preamble{ " \newcommand{\noop}[1]{} " }

@STRING{iclr="Proceedings of the International Conference on Learning Representations (ICLR)"}

@STRING{colt="Proceedings of the Conference on Learning Theory (COLT)"}

@STRING{icml="Proceedings of the International Conference on Machine Learning (ICML)"}

@STRING{neurips="Advances in Neural Information Processing Systems (NeurIPS)"}

@STRING{alt="International Conference on Algorithmic Learning Theory (ALT)"}

@STRING{tmlr="Transactions on Machine Learning Research (TMLR)"}

@STRING{emnlp="Proceedings of the Conference on Empirical Methods in Natural Language Processing (EMNLP)"}

@misc{kveton25active,
      title={Active Learning for Direct Preference Optimization}, 
      author={Branislav Kveton and Xintong Li and Julian McAuley and Ryan Rossi and Jingbo Shang and Junda Wu and Tong Yu},
      year={2025},
      eprint={2503.01076}
}

@inproceedings{hong24orpo,
  author       = {Jiwoo Hong and Noah Lee and James Thorne},
  title        = {{ORPO:} Monolithic Preference Optimization without Reference Model},
  booktitle    = emnlp,
  year         = {2024}
}

@InProceedings{ethayarajh25model,
  title = 	 {Model Alignment as Prospect Theoretic Optimization},
  author =       {Ethayarajh, Kawin and Xu, Winnie and Muennighoff, Niklas and Jurafsky, Dan and Kiela, Douwe},
  booktitle =  icml,
  year = 	 {2024}
}

@article{yuan23rrhf,
  title={Rrhf: Rank responses to align language models with human feedback},
  author={Yuan, Hongyi and Yuan, Zheng and Tan, Chuanqi and Wang, Wei and Huang, Songfang and Huang, Fei},
  journal=neurips,
  year={2023}
}

@article{amini24direct,
  author       = {Afra Amini and
                  Tim Vieira and
                  Ryan Cotterell},
  title        = {Direct Preference Optimization with an Offset},
  journal      = {CoRR},
  volume       = {abs/2402.10571},
  year         = {2024}
}

@inproceedings{park24disentangling,
  author       = {Ryan Park and
                  Rafael Rafailov and
                  Stefano Ermon and
                  Chelsea Finn},
  title        = {Disentangling Length from Quality in Direct Preference Optimization},
  booktitle    = {Findings of the Association for Computational Linguistics (ACL)},
  year         = {2024}
}

@InProceedings{tang24generalized,
  title = 	 {Generalized Preference Optimization: A Unified Approach to Offline Alignment},
  author =       {Tang, Yunhao and Guo, Zhaohan Daniel and Zheng, Zeyu and Calandriello, Daniele and Munos, Remi and Rowland, Mark and Richemond, Pierre Harvey and Valko, Michal and Avila Pires, Bernardo and Piot, Bilal},
  booktitle = icml,
  year = 	 {2024}
}

@article{zhao23slic-hf,
  publtype={informal},
  author={Yao Zhao and Rishabh Joshi and Tianqi Liu and Misha Khalman and Mohammad Saleh and Peter J. Liu},
  title={SLiC-HF: Sequence Likelihood Calibration with Human Feedback},
  year={2023},
  journal={CoRR},
  volume={abs/2305.10425},
}

@inproceedings{zhao23calibrating,
title={Calibrating Sequence likelihood Improves Conditional Language Generation},
author={Yao Zhao and Mikhail Khalman and Rishabh Joshi and Shashi Narayan and Mohammad Saleh and Peter J Liu},
booktitle=iclr,
year={2023}
}

@misc{schulman20kl,
  author={Schulman, John},
  url={http://joschu.net/blog/kl-approx.html},
  year={2020}
}

@book{ziebart10modeling,
  title={Modeling purposeful adaptive behavior with the principle of maximum causal entropy},
  author={Ziebart, Brian D},
  year={2010},
  publisher={Carnegie Mellon University}
}

@inproceedings{munos03error,
  title={Error bounds for approximate policy iteration},
  author={Munos, R{\'e}mi},
  booktitle=icml,
  year={2003},
}

@inproceedings{mao19mode,
  title={Mode seeking generative adversarial networks for diverse image synthesis},
  author={Mao, Qi and Lee, Hsin-Ying and Tseng, Hung-Yu and Ma, Siwei and Yang, Ming-Hsuan},
  booktitle={Proceedings of the IEEE/CVF conference on computer vision and pattern recognition},
  pages={1429--1437},
  year={2019}
}

@article{goodfellow14generative,
  title={Generative adversarial nets},
  author={Goodfellow, Ian J and Pouget-Abadie, Jean and Mirza, Mehdi and Xu, Bing and Warde-Farley, David and Ozair, Sherjil and Courville, Aaron and Bengio, Yoshua},
  journal=neurips,
  year={2014}
}

@InProceedings{zhan22offline,
  title = 	 {Offline Reinforcement Learning with Realizability and Single-policy Concentrability},
  author =       {Zhan, Wenhao and Huang, Baihe and Huang, Audrey and Jiang, Nan and Lee, Jason},
  booktitle = colt,
  year = 	 {2022}
}

@article{dumoulin23density,
  title={A density estimation perspective on learning from pairwise human preferences},
  author={Dumoulin, Vincent and Johnson, Daniel D and Castro, Pablo Samuel and Larochelle, Hugo and Dauphin, Yann},
  journal={arXiv preprint arXiv:2311.14115},
  year={2023}
}

@article{agarwal19reinforcement,
  title={Reinforcement learning: Theory and algorithms},
  author={Agarwal, Alekh and Jiang, Nan and Kakade, Sham M and Sun, Wen},
  journal={CS Dept., UW Seattle, Seattle, WA, USA, Tech. Rep},
  year={2019}
}

@inproceedings{xie21bellman,
 author = {Xie, Tengyang and Cheng, Ching-An and Jiang, Nan and Mineiro, Paul and Agarwal, Alekh},
 booktitle = neurips,
 title = {Bellman-consistent Pessimism for Offline Reinforcement Learning},
 year = {2021}
}

@article{bradley52rank,
  title={Rank analysis of incomplete block designs: I. The method of paired comparisons},
  author={Bradley, Ralph Allan and Terry, Milton E},
  journal={Biometrika},
  volume={39},
  number={3/4},
  pages={324--345},
  year={1952},
  publisher={JSTOR}
}

@article{gao24rebel,
  title={Rebel: Reinforcement learning via regressing relative rewards},
  author={Gao, Zhaolin and Chang, Jonathan and Zhan, Wenhao and Oertell, Owen and Swamy, Gokul and Brantley, Kiant{\'e} and Joachims, Thorsten and Bagnell, Drew and Lee, Jason D and Sun, Wen},
  journal=neurips,
  year={2024}
}

@InProceedings{xiong24iterative,
  title = 	 {Iterative Preference Learning from Human Feedback: Bridging Theory and Practice for {RLHF} under {KL}-constraint},
  author =       {Xiong, Wei and Dong, Hanze and Ye, Chenlu and Wang, Ziqi and Zhong, Han and Ji, Heng and Jiang, Nan and Zhang, Tong},
  booktitle = icml,
  year = 	 {2024}
}

@inproceedings{xie25exploratory,
title={Exploratory Preference Optimization: Harnessing Implicit Q*-Approximation for Sample-Efficient {RLHF}},
author={Tengyang Xie and Dylan J Foster and Akshay Krishnamurthy and Corby Rosset and Ahmed Hassan Awadallah and Alexander Rakhlin},
booktitle={International Conference on Learning Representations},
year={2025}
}

@inproceedings{huang25correcting,
title={Correcting the Mythos of {KL}-Regularization: Direct Alignment without Overoptimization via Chi-Squared Preference Optimization},
author={Audrey Huang and Wenhao Zhan and Tengyang Xie and Jason D. Lee and Wen Sun and Akshay Krishnamurthy and Dylan J Foster},
booktitle={International Conference on Learning Representations},
year={2025}
}

@article{grattafiori2024llama,
  title={The llama 3 herd of models},
  author={Grattafiori, Aaron and Dubey, Abhimanyu and Jauhri, Abhinav and Pandey, Abhinav and Kadian, Abhishek and Al-Dahle, Ahmad and Letman, Aiesha and Mathur, Akhil and Schelten, Alan and Vaughan, Alex and others},
  journal={arXiv preprint arXiv:2407.21783},
  year={2024}
}

@inproceedings{azar24general,
  title={A general theoretical paradigm to understand learning from human preferences},
  author={Azar, Mohammad Gheshlaghi and Guo, Zhaohan Daniel and Piot, Bilal and Munos, Remi and Rowland, Mark and Valko, Michal and Calandriello, Daniele},
  booktitle={International Conference on Machine Learning},
  year={2024},
}

@article{ouyang22training,
  title={Training language models to follow instructions with human feedback},
  author={Ouyang, Long and Wu, Jeffrey and Jiang, Xu and Almeida, Diogo and Wainwright, Carroll and Mishkin, Pamela and Zhang, Chong and Agarwal, Sandhini and Slama, Katarina and Ray, Alex and others},
  journal=neurips,
  year={2022}
}

@article{bai22training,
  title={Training a helpful and harmless assistant with reinforcement learning from human feedback},
  author={Bai, Yuntao and Jones, Andy and Ndousse, Kamal and Askell, Amanda and Chen, Anna and DasSarma, Nova and Drain, Dawn and Fort, Stanislav and Ganguli, Deep and Henighan, Tom and others},
  journal={arXiv preprint arXiv:2204.05862},
  year={2022}
}

@article{stiennon20learning,
  title={Learning to summarize with human feedback},
  author={Stiennon, Nisan and Ouyang, Long and Wu, Jeffrey and Ziegler, Daniel and Lowe, Ryan and Voss, Chelsea and Radford, Alec and Amodei, Dario and Christiano, Paul F},
  journal=neurips,
  year={2020}
}

@article{christiano17deep,
  title={Deep reinforcement learning from human preferences},
  author={Christiano, Paul F and Leike, Jan and Brown, Tom and Martic, Miljan and Legg, Shane and Amodei, Dario},
  journal=neurips,
  year={2017}
}

@article{agarwal25design,
  title={Design Considerations in Offline Preference-based RL},
  author={Agarwal, Alekh and Dann, Christoph and Marinov, Teodor V},
  journal={arXiv preprint arXiv:2502.06861},
  year={2025}
}

@article{meng24simpo,
  title={Simpo: Simple preference optimization with a reference-free reward},
  author={Meng, Yu and Xia, Mengzhou and Chen, Danqi},
  journal=neurips,
  volume={37},
  pages={124198--124235},
  year={2024}
}

@article{foster21efficient,
  title={Efficient first-order contextual bandits: Prediction, allocation, and triangular discrimination},
  author={Foster, Dylan J and Krishnamurthy, Akshay},
  journal=neurips,
  pages={18907--18919},
  year={2021}
}

@article{fisch25robust,
  title={Robust Preference Optimization through Reward Model Distillation},
  author={Adam Fisch and Jacob Eisenstein and Vicky Zayats and Alekh Agarwal and Ahmad Beirami and Chirag Nagpal and Peter Shaw and Jonathan Berant},
  journal=tmlr,
  issn={2835-8856},
  year={2025},
}

@article{song24theimportance,
  title={The importance of online data: Understanding preference fine-tuning via coverage},
  author={Song, Yuda and Swamy, Gokul and Singh, Aarti and Bagnell, J and Sun, Wen},
  journal=neurips,
  volume={37},
  pages={12243--12270},
  year={2024}
}

@article{guo24direct,
  title={Direct language model alignment from online ai feedback},
  author={Guo, Shangmin and Zhang, Biao and Liu, Tianlin and Liu, Tianqi and Khalman, Misha and Llinares, Felipe and Rame, Alexandre and Mesnard, Thomas and Zhao, Yao and Piot, Bilal and others},
  journal={arXiv preprint arXiv:2402.04792},
  year={2024}
}

@misc{swamy25all,
      title={All Roads Lead to Likelihood: The Value of Reinforcement Learning in Fine-Tuning}, 
      author={Gokul Swamy and Sanjiban Choudhury and Wen Sun and Zhiwei Steven Wu and J. Andrew Bagnell},
      year={2025},
      eprint={2503.01067},
      archivePrefix={arXiv}
}

@article{rafailov23direct,
  title={Direct preference optimization: Your language model is secretly a reward model},
  author={Rafailov, Rafael and Sharma, Archit and Mitchell, Eric and Manning, Christopher D and Ermon, Stefano and Finn, Chelsea},
  journal=neurips,
  volume={36},
  pages={53728--53741},
  year={2023}
}

@inproceedings{gabbianelli24importance,
  title={Importance-weighted offline learning done right},
  author={Gabbianelli, Germano and Neu, Gergely and Papini, Matteo},
  booktitle={International Conference on Algorithmic Learning Theory (ALT)},
  pages={614--634},
  year={2024},
  organization={PMLR}
}

@inproceedings{li24benchmarking,
  title={Benchmarking and Improving Generator-Validator Consistency of Language Models},
  author={Li, Xiang Lisa and Shrivastava, Vaishnavi and Li, Siyan and Hashimoto, Tatsunori and Liang, Percy},
  booktitle={The Twelfth International Conference on Learning Representations},
  year={2024}
}

@inproceedings{west24generative,
  title={The Generative AI Paradox:“What It Can Create, It May Not Understand”},
  author={West, Peter and Lu, Ximing and Dziri, Nouha and Brahman, Faeze and Li, Linjie and Hwang, Jena D and Jiang, Liwei and Fisher, Jillian and Ravichander, Abhilasha and Chandu, Khyathi and others},
  booktitle={The Twelfth International Conference on Learning Representations},
  year={2024}
}

@book{van2009empirical,
  title={Empirical Processes in M-Estimation},
  author={van de Geer, S.A.},
  series={Cambridge Series in Statistical and Probabilistic Mathematics},
  year={2009},
  publisher={Cambridge University Press}
}

@article{tong07entropy,
author = {Zhang, Tong},
year = {2007},
title = {From $\epsilon$-entropy to KL-entropy: Analysis of minimum information complexity density estimation},
volume = {34},
journal = {The Annals of Statistics}
}

@inproceedings{biderman2023pythia,
  title={Pythia: A suite for analyzing large language models across training and scaling},
  author={Biderman, Stella and Schoelkopf, Hailey and Anthony, Quentin Gregory and Bradley, Herbie and O’Brien, Kyle and Hallahan, Eric and Khan, Mohammad Aflah and Purohit, Shivanshu and Prashanth, USVSN Sai and Raff, Edward and others},
  booktitle={International Conference on Machine Learning},
  pages={2397--2430},
  year={2023},
  organization={PMLR}
}

@article{schulman2017proximal,
  title={Proximal policy optimization algorithms},
  author={Schulman, John and Wolski, Filip and Dhariwal, Prafulla and Radford, Alec and Klimov, Oleg},
  journal={arXiv preprint arXiv:1707.06347},
  year={2017}
}

@article{cui2023ultrafeedback,
  title={Ultrafeedback: Boosting language models with high-quality feedback},
  author={Cui, Ganqu and Yuan, Lifan and Ding, Ning and Yao, Guanming and Zhu, Wei and Ni, Yuan and Xie, Guotong and Liu, Zhiyuan and Sun, Maosong},
  journal={CoRR},
  year={2023}
}

@inproceedings{
xie2025exploratory,
title={Exploratory Preference Optimization: Harnessing Implicit Q*-Approximation for Sample-Efficient {RLHF}},
author={Tengyang Xie and Dylan J Foster and Akshay Krishnamurthy and Corby Rosset and Ahmed Hassan Awadallah and Alexander Rakhlin},
booktitle={The Thirteenth International Conference on Learning Representations},
year={2025},
url={https://openreview.net/forum?id=QYigQ6gXNw}
}

@article{ziegler2019fine,
  title={Fine-tuning language models from human preferences},
  author={Ziegler, Daniel M and Stiennon, Nisan and Wu, Jeffrey and Brown, Tom B and Radford, Alec and Amodei, Dario and Christiano, Paul and Irving, Geoffrey},
  journal={arXiv preprint arXiv:1909.08593},
  year={2019}
}

@misc{nakano2022webgpt,
      title={{WebGPT}: Browser-assisted question-answering with human feedback}, 
      author={Reiichiro Nakano and Jacob Hilton and Suchir Balaji and Jeff Wu and Long Ouyang and Christina Kim and Christopher Hesse and Shantanu Jain and Vineet Kosaraju and William Saunders and Xu Jiang and Karl Cobbe and Tyna Eloundou and Gretchen Krueger and Kevin Button and Matthew Knight and Benjamin Chess and John Schulman},
      year={2022},
      eprint={2112.09332},
      archivePrefix={arXiv},
      primaryClass={cs.CL},
      url={https://arxiv.org/abs/2112.09332}, 
}

@InProceedings{liang2024rich,
    author    = {Liang, Youwei and He, Junfeng and Li, Gang and Li, Peizhao and Klimovskiy, Arseniy and Carolan, Nicholas and Sun, Jiao and Pont-Tuset, Jordi and Young, Sarah and Yang, Feng and Ke, Junjie and Dvijotham, Krishnamurthy Dj and Collins, Katherine M. and Luo, Yiwen and Li, Yang and Kohlhoff, Kai J and Ramachandran, Deepak and Navalpakkam, Vidhya},
    title     = {Rich Human Feedback for Text-to-Image Generation},
    booktitle = {Proceedings of the IEEE/CVF Conference on Computer Vision and Pattern Recognition (CVPR)},
    month     = {June},
    year      = {2024},
    pages     = {19401-19411}
}

@misc{lee2023aligning,
      title={Aligning Text-to-Image Models using Human Feedback}, 
      author={Kimin Lee and Hao Liu and Moonkyung Ryu and Olivia Watkins and Yuqing Du and Craig Boutilier and Pieter Abbeel and Mohammad Ghavamzadeh and Shixiang Shane Gu},
      year={2023},
      eprint={2302.12192},
      archivePrefix={arXiv},
      primaryClass={cs.LG},
      url={https://arxiv.org/abs/2302.12192}, 
}

@misc{kaufmann2024survey,
      title={A Survey of Reinforcement Learning from Human Feedback}, 
      author={Timo Kaufmann and Paul Weng and Viktor Bengs and Eyke Hüllermeier},
      year={2024},
      eprint={2312.14925},
      archivePrefix={arXiv},
      primaryClass={cs.LG},
      url={https://arxiv.org/abs/2312.14925}, 
}

@misc{openai2022chatgpt,
  author       = {{OpenAI}},
  title        = {{ChatGPT: Optimizing Language Models for Dialogue}},
  howpublished = {\url{https://openai.com/blog/chatgpt}},
  year         = {2022},
  note         = {Accessed: 2025-05-19}
}

@inproceedings{
tunstall2024zephyr,
title={Zephyr: Direct Distillation of {LM} Alignment},
author={Lewis Tunstall and Edward Emanuel Beeching and Nathan Lambert and Nazneen Rajani and Kashif Rasul and Younes Belkada and Shengyi Huang and Leandro Von Werra and Cl{\'e}mentine Fourrier and Nathan Habib and Nathan Sarrazin and Omar Sanseviero and Alexander M Rush and Thomas Wolf},
booktitle={First Conference on Language Modeling},
year={2024}
}

@inproceedings{Cal-DPO2024,
  title={Cal-DPO: Calibrated Direct Preference Optimization for Language Model Alignment },
  author={Xiao, Teng and Yuan, Yige and Zhu, Huaisheng and Li, Mingxiao and Honavar, Vasant G},
  booktitle={The Thirty-eighth Annual Conference on Neural Information Processing Systems (NeurIPS)},
  year={2024}
}

@misc{wallace2023diffusion,
      title={Diffusion Model Alignment Using Direct Preference Optimization}, 
      author={Bram Wallace and Meihua Dang and Rafael Rafailov and Linqi Zhou and Aaron Lou and Senthil Purushwalkam and Stefano Ermon and Caiming Xiong and Shafiq Joty and Nikhil Naik},
      year={2023},
      eprint={2311.12908},
      archivePrefix={arXiv},
      primaryClass={cs.CV},
      url={https://arxiv.org/abs/2311.12908}, 
}

@misc{qi2024onlinedpo,
      title={Online DPO: Online Direct Preference Optimization with Fast-Slow Chasing}, 
      author={Biqing Qi and Pengfei Li and Fangyuan Li and Junqi Gao and Kaiyan Zhang and Bowen Zhou},
      year={2024},
      eprint={2406.05534},
      archivePrefix={arXiv},
      primaryClass={cs.AI},
      url={https://arxiv.org/abs/2406.05534}
}

@misc{liu2024iterative,
      title={Iterative Length-Regularized Direct Preference Optimization: A Case Study on Improving 7B Language Models to GPT-4 Level}, 
      author={Jie Liu and Zhanhui Zhou and Jiaheng Liu and Xingyuan Bu and Chao Yang and Han-Sen Zhong and Wanli Ouyang},
      year={2024},
      eprint={2406.11817},
      archivePrefix={arXiv},
      primaryClass={cs.CL},
      url={https://arxiv.org/abs/2406.11817}
}

@misc{nath2025simul,
      title={Simultaneous Reward Distillation and Preference Learning: Get You a Language Model Who Can Do Both}, 
      author={Abhijnan Nath and Changsoo Jung and Ethan Seefried and Nikhil Krishnaswamy},
      year={2025},
      eprint={2410.08458},
      archivePrefix={arXiv},
      primaryClass={cs.LG},
      url={https://arxiv.org/abs/2410.08458}, 
}

@misc{zhang2025distilldatarewardssmaller,
      title={Distill Not Only Data but Also Rewards: Can Smaller Language Models Surpass Larger Ones?}, 
      author={Yudi Zhang and Lu Wang and Meng Fang and Yali Du and Chenghua Huang and Jun Wang and Qingwei Lin and Mykola Pechenizkiy and Dongmei Zhang and Saravan Rajmohan and Qi Zhang},
      year={2025},
      eprint={2502.19557},
      archivePrefix={arXiv},
      primaryClass={cs.CL},
      url={https://arxiv.org/abs/2502.19557}, 
}

@inproceedings{
huang2024the,
title={The N+ Implementation Details of {RLHF} with {PPO}: A Case Study on {TL};{DR} Summarization},
author={Shengyi Huang and Michael Noukhovitch and Arian Hosseini and Kashif Rasul and Weixun Wang and Lewis Tunstall},
booktitle={First Conference on Language Modeling},
year={2024},
url={https://openreview.net/forum?id=kHO2ZTa8e3}
}

@inproceedings{
hu2022lora,
title={Lo{RA}: Low-Rank Adaptation of Large Language Models},
author={Edward J Hu and yelong shen and Phillip Wallis and Zeyuan Allen-Zhu and Yuanzhi Li and Shean Wang and Lu Wang and Weizhu Chen},
booktitle={International Conference on Learning Representations},
year={2022},
url={https://openreview.net/forum?id=nZeVKeeFYf9}
}

@inproceedings{
zhan2024provable,
title={Provable Offline Preference-Based Reinforcement Learning},
author={Wenhao Zhan and Masatoshi Uehara and Nathan Kallus and Jason D. Lee and Wen Sun},
booktitle={The Twelfth International Conference on Learning Representations},
year={2024},
url={https://openreview.net/forum?id=tVMPfEGT2w}
}

@inproceedings{
cen2025valueincentivized,
title={Value-Incentivized Preference Optimization: A Unified Approach to Online and Offline {RLHF}},
author={Shicong Cen and Jincheng Mei and Katayoon Goshvadi and Hanjun Dai and Tong Yang and Sherry Yang and Dale Schuurmans and Yuejie Chi and Bo Dai},
booktitle={The Thirteenth International Conference on Learning Representations},
year={2025},
url={https://openreview.net/forum?id=SQnitDuow6}
}

@misc{zhang2024self,
      title={Self-Exploring Language Models: Active Preference Elicitation for Online Alignment}, 
      author={Shenao Zhang and Donghan Yu and Hiteshi Sharma and Han Zhong and Zhihan Liu and Ziyi Yang and Shuohang Wang and Hany Hassan and Zhaoran Wang},
      year={2024},
      eprint={2405.19332},
      archivePrefix={arXiv},
      primaryClass={cs.LG},
      url={https://arxiv.org/abs/2405.19332}, 
}

@misc{foster2025good,
      title={Is a Good Foundation Necessary for Efficient Reinforcement Learning? The Computational Role of the Base Model in Exploration}, 
      author={Dylan J. Foster and Zakaria Mhammedi and Dhruv Rohatgi},
      year={2025},
      eprint={2503.07453},
      archivePrefix={arXiv},
      primaryClass={cs.LG},
      url={https://arxiv.org/abs/2503.07453}, 
}

@inproceedings{
dubois2024lengthcontrolled,
title={Length-Controlled AlpacaEval: A Simple Debiasing of Automatic Evaluators},
author={Yann Dubois and Percy Liang and Tatsunori Hashimoto},
booktitle={First Conference on Language Modeling},
year={2024},
url={https://openreview.net/forum?id=CybBmzWBX0}
}

@inproceedings{li2025crowdsourced,
  title={From Crowdsourced Data to High-quality Benchmarks: Arena-Hard and Benchbuilder Pipeline},
  author={Li, Tianle and Chiang, Wei-Lin and Frick, Evan and Dunlap, Lisa and Wu, Tianhao and Zhu, Banghua and Gonzalez, Joseph E and Stoica, Ion},
  booktitle={International Conference on Machine Learning},
  pages={34209--34231},
  year={2025},
  organization={PMLR}
}

@article{zheng2023judging,
  title={Judging llm-as-a-judge with mt-bench and chatbot arena},
  author={Zheng, Lianmin and Chiang, Wei-Lin and Sheng, Ying and Zhuang, Siyuan and Wu, Zhanghao and Zhuang, Yonghao and Lin, Zi and Li, Zhuohan and Li, Dacheng and Xing, Eric and others},
  journal={Advances in Neural Information Processing Systems},
  volume={36},
  pages={46595--46623},
  year={2023}
}

@article{wang2024interpretable,
  title={Interpretable preferences via multi-objective reward modeling and mixture-of-experts},
  author={Wang, Haoxiang and Xiong, Wei and Xie, Tengyang and Zhao, Han and Zhang, Tong},
  journal={arXiv preprint arXiv:2406.12845},
  year={2024}
}
